\documentclass[11pt,letterpaper]{article}

\usepackage{lipsum}

\newcommand\blfootnote[1]{%
  \begingroup
  \renewcommand\thefootnote{}\footnote{#1}%
  \addtocounter{footnote}{-1}%
  \endgroup
}

\usepackage{fullpage}
\usepackage[utf8]{inputenc} 
\usepackage[T1]{fontenc}    
\usepackage{hyperref}       
\usepackage{url}            
\usepackage{amsfonts}       
\usepackage{nicefrac}       
\usepackage{microtype}      

\usepackage{xcolor}
\usepackage{amsmath}
\usepackage{amssymb}
\usepackage{amsthm}
\usepackage{algorithm}
\usepackage{algpseudocode}
\usepackage{hyperref}
\usepackage{graphicx}
\usepackage[font={small}]{caption}
\usepackage{thm-restate}
\usepackage[normalem]{ulem}
\usepackage{xspace}
\usepackage{arydshln}
\usepackage{wrapfig}
\usepackage{enumitem}
\usepackage{framed}
\usepackage{dsfont}

\allowdisplaybreaks

\algnotext{EndFor}
\algnotext{EndIf}
\algnotext{EndWhile}
\algnotext{EndProcedure}

\newcommand{\A}{\mathcal{A}}
\newcommand{\B}{\mathcal{B}}
\newcommand{\E}{\mathbb{E}}
\newcommand{\bw}{\mathbf{w}}
\newcommand{\bzero}{\mathbf{0}}
\newcommand{\hw}{\hat{w}}
\newcommand{\bd}{\mathbf{v}}
\newcommand{\bb}{\mathbf{b}}
\newcommand{\bt}{\mathbf{t}}
\newcommand{\bz}{\mathbf{z}}
\newcommand{\bhw}{\hat{\mathbf{w}}}
\newcommand{\bv}{\mathbf{v}}
\newcommand{\bu}{\mathbf{u}}
\newcommand{\bx}{\mathbf{x}}
\newcommand{\be}{\mathbf{e}}
\newcommand{\bs}{\mathbf{s}}
\newcommand{\cD}{\mathcal{D}}
\newcommand{\cO}{\mathcal{O}}
\newcommand{\bB}{\mathbb{B}}
\newcommand{\N}{\mathbb{N}}
\newcommand{\R}{\mathbb{R}}
\newcommand{\cL}{\mathcal{L}}
\newcommand{\cU}{\mathcal{U}}
\newcommand{\bone}{\mathds{1}}
\newcommand{\eps}{\epsilon}
\newcommand{\teps}{\tilde{\epsilon}}
\newcommand{\td}{\tilde{d}}
\newcommand{\tS}{\tilde{S}}
\newcommand{\tbw}{\tilde{\bw}}
\newcommand{\tPi}{\tilde{\Pi}}
\newcommand{\hPi}{\hat{\Pi}}

\DeclareMathOperator{\sgn}{sgn}
\DeclareMathOperator{\err}{err}
\DeclareMathOperator{\opt}{OPT}
\DeclareMathOperator{\val}{val}
\DeclareMathOperator{\wval}{wval}
\DeclareMathOperator{\poly}{poly}

\newcommand{\1}{\mathds{1}}

\newtheorem{theorem}{Theorem}

\newtheorem{observation}[theorem]{Observation}
\newtheorem{hypothesis}[theorem]{Hypothesis}
\newtheorem{lemma}[theorem]{Lemma}

\newtheorem{definition}[theorem]{Definition}

\newtheorem{proposition}[theorem]{Proposition}
\newtheorem{fact}[theorem]{Fact}

\title{The Complexity of Adversarially Robust \\ 
Proper Learning of Halfspaces with Agnostic Noise\blfootnote{Authors are in alphabetical order.}}

\author{
  Ilias Diakonikolas\thanks{Supported by NSF Award CCF-1652862 (CAREER) and a Sloan Research Fellowship.} \\
  University of Wisconsin, Madison\\
  {\tt ilias@cs.wisc.edu} \\
  \and
  Daniel M. Kane\thanks{Supported by NSF Award CCF-1553288 (CAREER) and a Sloan Research Fellowship.} \\
  University of California, San Diego\\
  {\tt dakane@cs.ucsd.edu} \\
  \and
  Pasin Manurangsi \\
  Google Research, Mountain View\\
 {\tt pasin@google.com}
}

\begin{document}

\maketitle

\begin{abstract}
We study the computational complexity of adversarially robust proper learning of halfspaces in the distribution-independent agnostic PAC model, with a focus on $L_p$ perturbations. We give a computationally efficient learning algorithm and a nearly matching computational hardness result for this problem. An interesting implication of our findings is that the $L_{\infty}$ perturbations case
is provably computationally harder than the case $2 \leq p < \infty$.
\end{abstract}

\thispagestyle{empty}

\setcounter{page}{0}

\newpage

\section{Introduction} \label{sec:intro}

In recent years, the design of reliable machine learning systems 
for secure-critical applications, including in computer vision and natural language processing, 
has been a major goal in the field. One of the main concrete goals in this context
has been to develop classifiers that are robust to {\em adversarial examples}, i.e., small 
imperceptible perturbations to the input that can result in erroneous 
misclassification~\cite{BiggioCMNSLGR13, SzegedyZSBEGF13, GoodfellowSS14}.
This has led to an explosion of research on designing defenses against adversarial 
examples and attacks on these defenses. See, e.g.,~\cite{KM-tutorial} for a recent tutorial
on the topic. Despite significant empirical progress over the past few years, 
the broad question of designing computationally efficient classifiers that are 
provably robust to adversarial perturbations remains an outstanding theoretical challenge.

In this paper, we focus on understanding the {\em computational complexity} of adversarially robust
classification in the (distribution-independent) agnostic PAC model~\cite{Haussler:92, KSS:94}.
Specifically, we study the learnability  of {\em halfspaces} (or linear threshold functions) in this model with respect to $L_p$ perturbations. 
A halfspace is any 
function $h_{\bw}: \R^d \to \{ \pm 1\}$ of the form\footnote{The function $\sgn: \R \to \{ \pm 1\}$ is defined as $\sgn(u)=1$ if $u \geq 0$ and $\sgn(u)=-1$ otherwise.} 
$h_{\bw}(\bx) = \sgn \left(\langle \bw, \bx \rangle \right)$, where $\bw \in \R^d$ is the associated weight vector.
The problem of learning an unknown halfspace has been studied for decades --- starting 
with the Perceptron algorithm~\cite{Rosenblatt:58} --- and has arguably been one of 
the most influential problems in the development of machine learning~\cite{Vapnik:98, FreundSchapire:97}.

Before we proceed, we introduce the relevant terminology.
Let $\mathcal{C}$ be a concept class of Boolean-valued functions 
on an instance space $\mathcal{X} \subseteq \R^d$ and 
$\mathcal{H}$ be a hypothesis class on $\mathcal{X}$. The set of allowable perturbations is 
defined by a function $\mathcal{U}: \mathcal{X} \to 2^{\mathcal{X}}$. 
The robust risk of a hypothesis $h \in \mathcal{H}$ with respect to a distribution 
$\mathcal{D}$ on $\mathcal{X} \times \{ \pm 1\}$ is defined as 
$\mathcal{R}_{\mathcal{U}}(h, \mathcal{D}) = 
\Pr_{(\bx, y) \sim \mathcal{D}}[\exists z \in \mathcal{U}(\bx), h(\mathbf{z}) \neq y].$
The (adversarially robust) agnostic PAC learning problem for $\mathcal{C}$ is the following:
Given i.i.d. samples from an arbitrary distribution $\mathcal{D}$ on $\mathcal{X} \times \{ \pm 1\}$, 
the goal of the learner is to output a hypothesis $h \in \mathcal{H}$ such that with high probability it holds 
$\mathcal{R}_{\mathcal{U}}(h, \mathcal{D}) \leq \opt+\eps$,
where $\opt = \inf_{f \in \mathcal{C}} \mathcal{R}_{\mathcal{U}}(f, \mathcal{D})$ is the robust risk of the best-fitting function in $\mathcal{C}$.

Unfortunately, it follows from known hardness results that 
this formulation is computationally intractable for the class
of halfspaces $\mathcal{C} = \{ \sgn(\langle \bw, \bx \rangle), \bw \in \R^d  \}$
under $L_p$ perturbations, i.e,  for $\mathcal{U}_{p, \gamma}(\bx) = \{ \mathbf{z} \in \mathcal{X}: \|\mathbf{z}-\bx\|_p \leq \gamma \}$, for some $p \geq 2$. (The reader is referred to Appendix~\ref{app:agnostic-hard} for an explanation.)
To be able to obtain computationally efficient algorithms, we relax the above definition in two ways:
(1) We allow the hypothesis to be robust within a slightly smaller perturbation region, 
and (2) We introduce a small constant factor approximation in the error guarantee. In more detail,
for some constants $0< \nu <1$ and $\alpha>1$,
our goal is to efficiently compute a hypothesis $h$ such that with high probability
\begin{equation} \label{eqn:bicriterion-risk}
\mathcal{R}_{\mathcal{U}_{p, (1-\nu) \gamma}}(h, \mathcal{D}) \leq \alpha \cdot \opt_{p, \gamma}+\eps \;,
\end{equation}
where $\opt_{p, \gamma} = \inf_{f \in \mathcal{C}} \mathcal{R}_{\mathcal{U}_{p,\gamma}}(f, \mathcal{D})$.
(Note that for $\nu = 0$ and $\alpha = 1$, we obtain the original definition.) An interesting setting
is when $\nu$ is a small constant close to $0$, say $\nu = 0.1$, and $\alpha = 1+\delta$, where
$0< \delta<1$. In this paper, we characterize the computational complexity of this problem
with respect to {\em proper} learning algorithms, i.e., algorithms that output
a halfspace hypothesis. 

Throughout this paper, we will assume that the domain of our functions 
is bounded in the $d$-dimensional $L_p$ unit ball $\bB_p^d$. All our results immediately extend to general domains with a (necessary) dependence on the diameter of the feasible set.



A simple but crucial observation leveraged in our work is the following:
The adversarially robust learning problem of halfspaces under $L_p$ perturbations (defined above) is essentially equivalent to the classical 
problem of agnostic proper PAC learning of halfspaces with an $L_p$ margin.

Let $p \geq 2$, $q$ be the dual exponent of $p$, i.e., $1/p+1/q=1$.
The problem of agnostic proper PAC learning of halfspaces with an $L_p$ margin is the following:
The learner is given i.i.d. samples from a distribution 
$\cD$ over $\bB_p^d \times \{\pm 1\}$.
For $\bw \in \bB_q^d$, its \emph{$\gamma$-margin error} is defined as $\err_\gamma^{\cD}(\bw) := \Pr_{(\bx, y) \sim \cD}[\sgn(\left<\bw, \bx\right> - y \cdot \gamma) \ne y]$. We also define $\opt_\gamma^\cD := \min_{\bw \in \bB_q^d} \err_\gamma^{\cD}(\bw)$.
An algorithm is a proper \emph{$\nu$-robust $\alpha$-agnostic learner} for $L_p$-$\gamma$-margin halfspace if, with probability at least $1 - \tau$, it outputs a halfspace $\bw \in \bB_q^d$ with 
\begin{equation} \label{eqn:agnostic-margin}
\err_{(1 - \nu)\gamma}^{\cD}(\bw) \leq \alpha \cdot \opt^{\cD}_\gamma + \eps \;.
\end{equation}
(When unspecified, the failure probability $\tau$ is assumed to be 1/3. It is well-known and easy 
to see that we can always achieve arbitrarily small value of $\tau$ at the cost of $O(\log(1/\tau))$ multiplicative factor in the running time and sample complexity.)

We have the following basic observation, which implies that the learning objectives \eqref{eqn:bicriterion-risk} and \eqref{eqn:agnostic-margin} are equivalent. Throughout this paper, we will state our contributions using the margin formulation \eqref{eqn:agnostic-margin}.
\begin{fact} \label{fact:adversarial-robustness-vs-margin}
For any non-zero $\bw \in \R^d$, $\gamma \geq 0$ and $\cD$ over $\R^d \times \{\pm 1\}$, $\mathcal{R}_{\mathcal{U}_{p,\gamma}}(h_{\bw}, \mathcal{D}) = \err^{\cD}_{\gamma}(\frac{\bw}{\|\bw\|_q})$.
\end{fact}


\subsection{Our Contributions} \label{sec:results}

Our main positive result is a robust and agnostic proper learning algorithm for $L_p$-$\gamma$-margin halfspace with near-optimal running time:

\begin{theorem}[Robust Learning Algorithm] \label{thm:learning-algo-main}
Fix $2 \leq p < \infty$ and $0< \gamma <1$. For any $0<\nu, \delta<1$,
there is a proper $\nu$-robust $(1 + \delta)$-agnostic learner for $L_p$-$\gamma$-margin halfspace that draws $O(\frac{p}{\eps^2 \nu^2 \gamma^2})$ samples and runs in time $(1/\delta)^{O\left(\frac{p}{\nu^2 \gamma^2}\right)} \cdot \poly(d/\eps)$.

Furthermore, for $p = \infty$, there is a proper $\nu$-robust $(1 + \delta)$-agnostic learner for $L_{\infty}$-$\gamma$-margin halfspace that draws $O(\frac{\log d}{\eps^2 \nu^2 \gamma^2})$ samples and runs in time $d^{O\left(\frac{\log(1/\delta)}{\nu^2 \gamma^2}\right)} \cdot \poly(1/\eps)$.
\end{theorem}


To interpret the running time of our algorithm, we consider the setting $\delta = \nu = 0.1$. We note two different regimes. If $p \geq 2$ is a fixed constant, then our algorithm runs in time $2^{O(1/\gamma^2)} \poly(d/\eps)$. On the other hand, for $p = \infty$, we obtain a runtime of $d^{O(1/\gamma^2)} \poly(1/\eps)$.
That is, the $L_{\infty}$ margin case (which corresponds to adversarial learning with $L_{\infty}$ perturbations)
appears to be computationally the hardest. 
As we show in Theorem~\ref{thm:running-time-lower-bound}, this fact is inherent for proper learners.

Our algorithm establishing Theorem~\ref{thm:learning-algo-main} follows via a simple and unified
approach, employing a reduction from online (mistake bound) learning
~\cite{Lit87}. Specifically, we show that any computationally efficient $L_p$ online learner for halfspaces with margin guarantees and mistake bound $M$ can be used in a black-box manner to obtain an algorithm for our problem with runtime roughly $\poly(d/\eps) (1/\delta)^{M}$. Theorem~\ref{thm:learning-algo-main} then follows by applying known results from the online learning literature~\cite{Gentile01}.

For the special case of $p=2$ (and $\nu = 0.1$), recent work~\cite{DiakonikolasKM19} gave a sophisticated algorithm for our problem with running time $\poly(d/\eps) 2^{\tilde{O}(1/(\delta \gamma^2))}$. We note that our algorithm has significantly better dependence on the parameter $\delta$ (quantifying the approximation ratio), and better dependence on $1/\gamma$. Importantly, our algorithm is much simpler and immediately generalizes to all $L_p$ norms.

Perhaps surprisingly, the running time of our algorithm is nearly the best possible for proper learning. For constant $p \geq 2$, this follows 
from the hardness result of~\cite{DiakonikolasKM19}. Furthermore, we prove a tight running time lower bound 
for robust $L_{\infty}$-$\gamma$-margin proper learning of halfspaces. 
Roughly speaking, we show that for some sufficiently small constant $\nu > 0$, one cannot hope to significantly speed-up our algorithm for $\nu$-robust $L_{\infty}$-$\gamma$-margin learning of halfspaces. Our computational hardness result is formally stated below.

\begin{theorem}[Tight Running Time Lower Bound] \label{thm:running-time-lower-bound}
There exists a constant $\nu > 0$ such that, assuming the (randomized) Gap Exponential Time Hypothesis
, 
there is no proper $\nu$-robust 1.5-agnostic learner for $L_{\infty}$-$\gamma$-margin halfspace that runs in time $f(1/\gamma) \cdot d^{o(1 / \gamma^2)} \poly(1/\eps)$ for any function $f$. 
\end{theorem}

As indicated above, our running time lower bound is based on the so-called Gap Exponential Time Hypothesis (Gap-ETH), 
which roughly states that no subexponential time algorithm can approximate 3SAT to within $(1 - \eps)$ factor, 
for some constant $\eps > 0$. Since we will not be dealing with Gap-ETH directly here, 
we defer the formal treatment of the hypothesis and discussions on its application to Section~\ref{subsec:eth}.

We remark that the constant $1.5$ in our theorem is insignificant. 
We can increase this ``gap'' to any constant less than 2. 
We use the value $1.5$ to avoid introducing an additional variable. 
Another remark is that Theorem~\ref{thm:running-time-lower-bound} only applies for a small constant $\nu > 0$. 
This leaves the possibility of achieving, e.g., a faster 0.9-robust $L_{\infty}$-$\gamma$-margin learner for halfspaces, 
as an interesting open problem. 

\subsection{Related Work}
A sequence of recent works~\cite{CullinaBM18, SchmidtSTTM18, BubeckLPR19, MontasserHS19}
has studied the sample complexity of adversarially robust PAC learning for general concept
classes of bounded VC dimension and for halfspaces in particular. 
\cite{MontasserHS19} established an upper bound on the sample complexity of PAC learning  any concept class with finite VC dimension. 
A common implication of the aforementioned works is that, for some concept classes, 
the sample complexity of adversarially robust PAC learning 
is higher than the sample complexity of (standard) PAC learning.
For the class of halfspaces, which is the focus of the current paper,  
the sample complexity of adversarially robust agnostic PAC learning was shown to be 
essentially the same as that of (standard) agnostic PAC learning
~\cite{CullinaBM18, MontasserHS19}.

Turning to computational aspects,~\cite{BubeckLPR19, DegwekarNV19} showed that there exist classification tasks
that are efficiently learnable in the standard PAC model, but are computationally hard in the adversarially robust
setting (under cryptographic assumptions). Notably, the classification problems shown hard are 
artificial, in the sense that they do not correspond to natural concept classes.
\cite{AwasthiDV19} shows that adversarially robust proper learning of degree-$2$ polynomial threshold
functions is computationally hard, even in the realizable setting. On the positive side, \cite{AwasthiDV19} gives a polynomial-time
algorithm for adversarially robust learning of halfspaces under $L_{\infty}$ perturbations, again in the realizable setting.
More recently,~\cite{MonGDS20} generalized this upper bound to a broad class of perturbations, including $L_p$ perturbations. Moreover,~\cite{MonGDS20} gave an efficient algorithm for learning halfspaces with random 
classification noise~\cite{AL88}. We note that all these algorithms are proper.

The problem of agnostically learning halfspaces with a margin
has been studied extensively. A number of prior works~\cite{BenDavidS00, SSS09, SSS10, LS:11malicious, BirnbaumS12, DiakonikolasKM19} studied the case of $L_2$ margin and gave a range of time-accuracy tradeoffs for the problem.
The most closely related prior work is the recent work~\cite{DiakonikolasKM19}, which gave a proper
$\nu$-robust $\alpha$-agnostic learning for $L_2$-$\gamma$-margin halfspace with near-optimal running time when $\alpha, \nu$ are universal constants, and
a nearly matching computational hardness result. The algorithm of the current paper 
broadly generalizes, simplifies, and improves the algorithm of~\cite{DiakonikolasKM19}.

\subsection{Organization}
We describe our algorithm and prove Theorem~\ref{thm:learning-algo-main} in Section~\ref{sec:algo}. In Section~\ref{sec:prelim}, we provide further preliminaries needed for our lower bound proof. We then prove our main hardness result (Theorem~\ref{thm:running-time-lower-bound}) in Section~\ref{sec:hardness}. Finally, we conclude with open questions in Section~\ref{sec:open-q}.

\section{Upper Bound: From Online to Adversarially Robust Agnostic Learning}
\label{sec:algo}
In this section, we provide a generic method that turns an online (mistake bound) learning algorithm for halfspaces into an adversarially robust agnostic algorithm, which is then used to prove Theorem~\ref{thm:learning-algo-main}. 

Recall that online learning~\cite{Lit87} proceeds 
in a sequence of rounds. In each round, the algorithm is given an example point, 
produces a binary prediction on this point, and receives feedback on its prediction 
(after which it is allowed to update its hypothesis). The mistake bound of an online learner 
is the maximum number of mistakes (i.e., incorrect predictions) it can make over 
all possible sequences of examples.

We start by defining the notion of online learning with a margin gap in the context of halfspaces:

\begin{definition}
An online learner $\mathcal{A}$ for the class of halfspaces is called an $L_p$ online
learner with mistake bound $M$ and $(\gamma,\gamma’)$ margin gap
if it satisfies the following: In each round, $\mathcal{A}$ returns a vector $\bw \in \bB^d_q$. Moreover, for any sequence of labeled examples $(\bx_i, y_i)$ such that there exists $\bw^\ast \in \bB^d_q$ with $\sgn(\left<\bw^\ast, \bx_i\right> - y_i \gamma) = y_i$ for all $i$, there are at most $M$ values of $t$ such that $\sgn(\left<\bw_t, \bx_t\right> - y_t \gamma’) \ne y_t$, where
$\bw_t = \A((\bx_1, y_1),\ldots,(\bx_{t-1},y_{t-1}))$.
\end{definition}

The $L_p$ online learning problem of halfspaces has been studied extensively in the literature, see, e.g.,~\cite{Lit87, GroveLS01, Gentile:01, Gentile03a, BalcanB14}. We will use a result of~\cite{Gentile01}, which gives a polynomial time $L_p$ online learner with margin gap $(\gamma,(1-\nu)\gamma)$ and mistake bound $O((p-1)/\nu^2 \gamma^2)$.

We are now ready to state our generic proposition that translates an online algorithm with a given mistake bound into an agnostic learning algorithm. We will use the following notation:
For $S \subseteq \bB_p^d \times \{\pm 1\}$, we will use $S$ instead of $\cD$ to denote the empirical error on the uniform distribution over $S$. In particular, we denote $\err_\gamma^S(\bw) := \frac{1}{|S|} \cdot |\{(\bx, y) \in S \mid \sgn(\left<\bw, \bx\right> - y \gamma) \ne y\}|$.

The main result of this section is the following proposition.
While we state our proposition for the empirical error, it is simple to convert it into a generalization bound as we will show later in the proof of Theorem~\ref{thm:learning-algo-main}.

\begin{proposition} \label{prop:empirical-learner}
Assume that there is a polynomial time $L_p$ online learner $\A$ 
for halfspaces with a $(\gamma,\gamma’)$ margin gap and mistake bound of $M$. Then there exists an algorithm that given a multiset of labeled examples $S \subseteq \bB_p^d\times \{\pm 1\}$ and $\delta \in (0, 1)$, runs in $\poly(|S|d) \cdot 2^{O(M\log(1/\delta))}$ time and with probability $9/10$ returns $\bw \in \bB^d_q$ such that $\err_{\gamma’}^S(\bw) \leq (1 + \delta) \cdot \opt^S_{\gamma}$.
\end{proposition}

Notice that our algorithm runs in time $\poly(|S|d) \cdot 2^{O(M\log(1/\delta))}$ and has success probability $9/10$. 
It is more convenient to describe a version of our algorithm that runs in 
$\poly(|S|d)$ time, but has small success probability of $2^{-O(M\log(1/\delta))}$, as encapsulated by the following lemma.

\begin{lemma} \label{lem:empirical-learner-low-prob}
Assume that there is a polynomial time $L_p$ online learner $\A$ 
for halfspaces with a $(\gamma,\gamma’)$ margin gap and mistake bound of $M$. Then there exists an algorithm that given a multiset of labeled examples $S \subseteq \bB_p^d\times \{\pm 1\}$ and $\delta \in (0, 1)$, runs in $\poly(|S|d)$ time and with probability $2^{-O(M\log(1/\delta))}$ returns $\bw \in \bB^d_q$ such that $\err_{\gamma’}^S(\bw) \leq (1 + \delta) \cdot \opt^S_{\gamma}$.
\end{lemma}

Before proving Lemma~\ref{lem:empirical-learner-low-prob}, notice that Proposition~\ref{prop:empirical-learner} now follows by running the algorithm from Lemma~\ref{lem:empirical-learner-low-prob} independently $2^{O(M\log(1/\delta))}$ times and returning the $\bw$ with minimum $\err^S_{\gamma'}(\bw)$. Since each iteration has a $2^{-O(M\log(1/\delta))}$ probability of returning a $\bw$ with $\err^S_{\gamma'}(\bw) \leq (1 + \delta) \cdot \opt^S_{\gamma}$, with $90\%$ probability at least one of our runs finds a $\bw$ that satisfies this.

\begin{proof}[Proof of Lemma~\ref{lem:empirical-learner-low-prob}]
Let $\bw^* \in \bB_q^d$ denote an ``optimal'' halfspace with $\err^S_{\gamma}(\bw^*) = \opt^S_{\gamma}$.

The basic idea of the algorithm is to repeatedly run $\A$ on larger and larger subsets of samples each time adding one additional sample in $S$ that the current hypothesis gets wrong. The one worry here is that some of the points in $S$ might be errors, inconsistent with the true classifier $\bw^\ast$, and feeding them to our online learner will lead it astray. However, at any point in time, either we misclassify (w.r.t. margin $\gamma'$) only a $(1 + \delta) \cdot \opt^S_{\gamma}$ fraction of points (in which case we can abort early and use this hypothesis) or guessing a random misclassified point will have at least an $\Omega(\delta)$ probability of giving us a non-error. Since our online learner has a mistake bound of $M$, we will never need to make more than this many correct guesses. Specifically, the algorithm is as follows:
\begin{itemize}
\item Let $\mathrm{Samples}=\emptyset$
\item For $i=0$ to $M$
\begin{itemize}
\item Let $\bw = \A(\mathrm{Samples})$
\item Let $T$ be the set of $(\bx, y)\in S$ so that $\sgn(\left<\bw, \bx\right> - y \gamma') \ne y$
\item If $T=\emptyset$, and otherwise with $50\%$ probability, return $\bw$
\item Draw $(\bx_i, y_i)$ uniformly at random from $T$, and add it to $\mathrm{Samples}$
\end{itemize}
\item Return $\bw$
\end{itemize}

To analyze this algorithm, let $S_{bad}$ be the set of $(\bx,y)\in S$ with $\sgn(\left<\bw^\ast, \bx\right> - y \gamma) \ne y$. Recall that by assumption $|S_{bad}| \leq \opt^S_{\gamma} \cdot |S|$. We claim that with probability at least $2^{-O(M \log(1/\delta))}$ our algorithm never adds an element of $S_{bad}$ to $\mathrm{Samples}$ and never returns a $\bw$ in the for loop for which $\err_{\gamma'}^S(\bw) > (1+\delta) \cdot \opt_\gamma^S$.
This is because during each iteration of the algorithm either:
\begin{enumerate}
\item $\err_{\gamma'}^S(\bw) > (1+\delta) \cdot \opt_\gamma^S$. In this case, there is a $50\%$ probability that we do not return $\bw$. If we do not return, then $|T|\geq (1 + \delta) \cdot |S_{bad}|$ so there is at least a $\frac{\delta}{1 + \delta} \geq \delta/2$ probability that the new element added to $\mathrm{Samples}$ is not in $S_{bad}$.
\item Or $\err_{\gamma'}^S(\bw) \leq (1+\delta) \cdot \opt_\gamma^S$. In this case, there is a $50\%$ probability of returning $\bw$.
\end{enumerate}
Hence, there is a $(\delta/4)^{M + 1} \geq 2^{-O(M\log(1/\delta))}$ probability of never adding an element of $S_{bad}$ to $\mathrm{Samples}$ or returning a $\bw$ in our for-loop with $\err_{\gamma'}^S(\bw) > (1+\delta) \cdot \opt_\gamma^S$. When this occurs, we claim that we output $\bw$ such that $\err_{\gamma'}^S(\bw) \leq (1+\delta) \cdot \opt_\gamma^S$.
This is because, if this were not the case, we must have reached the final statement at which point we have $\mathrm{Samples} = ((\bx_0,y_0),\ldots,(\bx_M,y_M))$, where each $(\bx_i,y_i)$ satisfies $\sgn(\left<\bw^\ast, \bx_i\right> - y_i \gamma) = y_i$ and $\sgn(\left<\bw_i, \bx_i\right> - y_i \gamma') \ne y_i$ with $\bw_i = \A((\bx_0,y_0),\ldots,(\bx_{i-1},y_{i-1}))$. But this violates the mistake bound of $M$.

Thus, we output $\bw$ such that $\err_{\gamma'}^S(\bw) \leq (1+\delta) \cdot \opt_\gamma^S$ with probability at least $2^{-O(M\log(1/\delta))}$.
%
%
\end{proof}

We will now show how Proposition~\ref{prop:empirical-learner} can be used to derive Theorem~\ref{thm:learning-algo-main}. As stated earlier, we will require the following mistake bound for online learning with a margin gap from~\cite{Gentile01}.

\begin{theorem}[\cite{Gentile01}] \label{thm:gentile-online-learning}
For any $2 \leq p < \infty$, there exists a polynomial time $L_p$ online learner with margin gap $(\gamma,(1 - \nu)\gamma)$ and mistake bound $O\left(\frac{(p - 1)}{\nu^2 \gamma^2}\right)$. Furthermore, there is a polynomial time $L_\infty$ online learner with margin gap $(\gamma,(1 - \nu)\gamma)$ and mistake bound $O\left(\frac{\log d}{\nu^2 \gamma^2}\right)$.
\end{theorem}

\begin{proof}[Proof of Theorem~\ref{thm:learning-algo-main}]
Our $\nu$-robust $(1 + \delta)$-agnostic learner for $L_p$-$\gamma$-margin halfspace works as follows. First, it draws the appropriate number of samples 
$m$ (as stated in Theorem~\ref{thm:learning-algo-main}) from $\cD$. Then, it runs the algorithm from Proposition~\ref{prop:empirical-learner} on these samples for margin gap $(\gamma, (1 - \nu/2)\gamma)$.

Let $M_p$ denote the error bound for $L_p$ online learning with margin gap $(\gamma,(1 - \nu/2)\gamma)$ given by Theorem~\ref{thm:gentile-online-learning}. Our entire algorithm runs in time $\poly(m) \cdot 2^{O(M_p \cdot \log(1/\delta))}$. It is simple to check that this results in the claimed running time.

As for the error guarantee, let $\bw \in \bB_q^d$ be the output halfspace. With probability 0.8, we have
\begin{align*}
\err_{(1 - \nu)\gamma}^{\cD}(\bw) &\leq \err_{(1 - \nu/2)\gamma}^{S}(\bw) + \eps/2
\leq (1 + \delta) \cdot \opt^S_{(1 - \nu/2)\gamma} + \eps/2
\leq (1 + \delta) \cdot \opt^{\cD}_{\gamma} + \eps,
\end{align*}
where the first and last inequalities follow from standard margin generalization bounds~\cite{BartlettM02, Koltchinskii2002, KakadeST08} and the second inequality follows from the guarantee of Proposition~\ref{prop:empirical-learner}.
\end{proof}

\section{Additional Background for Hardness Result}
\label{sec:prelim}
In this section, we provide additional preliminaries required for the proof of Theorem~\ref{thm:running-time-lower-bound}. Throughout the lower bound proof in the next section, we will sometimes view a vector $\bw \in \R^d$ naturally as a column matrix $\bw \in \R^{1 \times d}$; for example, we may write $\left<\bw, \bx\right> = \bw\bx^T$. Furthermore, for any positive integer $m$, we use $[m]$ to denote $\{1, \dots, m\}$. We also use $\be_i$ to denote the $i$-th vector in the standard basis (i.e., the vector with value one in the $i$-th coordinate and zero in the remaining coordinates). We extend this notation to a set $S$ of coordinates and use $\be_S$ to denote the indicator vector for $S$, i.e., $\be_S = \sum_{i \in S} \be_i$.

\subsection{Exponential Time Hypotheses}
\label{subsec:eth}

Recall that, in the 3-satisfiability (3SAT) problem, we are given a set of clauses, where each clause is an OR of at most three literals. The goal is to determine whether there exists an assignment that satisfies all clauses. The Exponential Time Hypothesis (ETH)~\cite{IP01,IPZ01} asserts that there is no sub-exponential time algorithm for 3SAT. ETH is of course a strengthening of the famous $P \ne NP$ assumption. In recent years, this assumption has become an essential part of modern complexity theory, as it allows one to prove tight running time lower bounds for many NP-hard and parameterized problems. See, e.g.,~\cite{LokshtanovMS11} for a survey on the topic.

For our lower bound, we use a strengthening of ETH, called Gap-ETH. Roughly speaking, Gap-ETH says that even finding an \emph{approximate} solution to 3SAT is hard. This is stated more precisely below:

\begin{hypothesis}[(Randomized) Gap Exponential Time Hypothesis (Gap-ETH)~\cite{Dinur16,MR17}] \label{hyp:gap-eth}
There exists a constant $\zeta > 0$ such that no randomized $2^{o(n)}$-time algorithm can, given a 3SAT instance on $n$ variables, distinguish between the following two cases correctly with probability $2/3$:
\begin{itemize}
\item (Completeness) There exists an assignment that satisfies all clauses.
\item (Soundness) Every assignment violates at least $\zeta$ fraction of the clauses.
\end{itemize}
\end{hypothesis}

Although proposed relatively recently, Gap-ETH is intimately related to a well-known open question whether linear size probabilistic checkable proofs exist for 3SAT; for more detail, please refer to the discussion in~\cite{Dinur16}. Gap-ETH has been used as a starting point for proving numerous tight running time lower bounds against approximation algorithms (e.g.,~\cite{Dinur16,MR17,BennettGS17,AggarwalS18,JainKR19}) and parameterized approximation algorithms (e.g.,~\cite{ChalermsookCKLM17,DinurM18,BhattacharyyaGS18,Cohen-AddadG0LL19}). Indeed, we will use one such result as a starting point of our hardness reduction.

\subsection{Hardness of Label Cover}
The main component of our hardness result will be a reduction from the \emph{Label Cover} 
problem\footnote{Label Cover is sometimes referred to as \emph{Projection Game}  or \emph{Two-Prover One-Round Game}.}, which is a classical problem in hardness of approximation literature that is widely used as a starting point for proving strong NP-hardness of approximation results (see, e.g., \cite{AroraBSS97,Hastad96,Hastad01,Feige98}).

\begin{definition}[Label Cover] \label{def:label-cover-full}
A \emph{Label Cover} instance $\cL = (U, V, E, \Sigma_U, \Sigma_V, \{\pi_e\}_{e \in \Sigma})$ consists of
\begin{itemize}
\item a bi-regular bipartite graph $(U, V, E)$, referred to as the \emph{constraint graph},
\item \emph{label sets} $\Sigma_U$ and $\Sigma_V$,
\item for every edge $e \in E$, a \emph{constraint} (aka \emph{projection}) $\pi_e: \Sigma_U \to \Sigma_V$.
\end{itemize}

A labeling of $\cL$ is a function $\phi: U \to \Sigma_U$. 
We say that $\phi$ \emph{covers} $v \in V$ if there exists 
$\sigma_v \in \Sigma_V$ such that\footnote{This is equivalent to 
$\pi_{(u_1, v)}(\phi(u_1)) = \pi_{(u_2, v)}(\phi(u_2))$ for all neighbors $u_1, u_2$ of $v$.} 
$\pi_{(u, v)}(\phi(u)) = \sigma_v$ for all\footnote{For every $a \in U \cup V$, 
we use $N(a)$ to denote the set of neighbors of $a$ (with respect to the graph $(U, V, E)$).} 
$u \in N(v)$. The \emph{value} $\phi$, denoted by $\val_{\cL}(\phi)$, 
is defined as the fraction of $v \in V$ covered by $\phi$. 
The value of $\cL$, denoted by $\val(\cL)$, is defined as $\max_{\phi: U \to \Sigma_U} \val(\phi)$.

Moreover, we say that $\phi$ \emph{weakly covers} $v \in V$ if there exist distinct neighbors $u_1, u_2$ 
of $v$ such that $\pi_{(u_1, v)}(\phi(u_1)) = \pi_{(u_2, v)}(\phi(u_2))$. The \emph{weak value} of $\phi$, 
denoted by $\wval(\phi)$, is the fraction of $v \in V$ weakly covered by $\phi$. 
The weak value of $\cL$, denoted by $\wval(\cL)$, is defined as $\max_{\phi: U \to \Sigma_U} \wval(\phi)$. 

For a Label Cover instance $\cL$, we use $k$ to denote $|U|$ and $n$ to denote 
$|U| \cdot |\Sigma_U| + |V| \cdot |\Sigma_V|$.
\end{definition}

The goal of Label Cover is to find an assignment with maximium value.

In our reduction, we will also need an additional notion of ``decomposability'' of a Label Cover instance. Roughly speaking, an instance is {\em decomposable} if we can partition $V$ 
into different parts such that each $u \in U$ has exactly 
one induced edge to the vertices in each part:
\begin{definition}[Decomposable Label Cover] \label{def:label-cover-decomposable-main-body}
A \emph{Label Cover} instance $\cL = (U, V, E, \Sigma_U, \Sigma_V, \{\pi_e\}_{e \in E})$ 
is said to be \emph{decomposable} if there exists a partition of $V$ into $V_1 \cup \cdots \cup V_t$ 
such that, for every $u \in U$ and $j \in [t]$, $|N(u) \cap V_j| = 1$. 
%
We use the notation $v^j(u)$ to the denote 
the unique element in $N(u) \cap V_j$.
\end{definition}

Several strong inapproximability results for Label Cover 
are known~\cite{Raz98,MoshkovitzR10,DinurS14}. To prove a tight 
running time lower bound, we require an inapproximability result for Label Cover 
with a tight running lower bound as well.
Observe that we can solve Label Cover in time $n^{O(k)}$ by enumerating through all possible 
$|\Sigma_U|^{|U|} = n^{O(k)}$ assignments and compute their values. The following result shows that, even if we only aim for a constant approximation ratio, no algorithm that can be significantly faster than this ``brute-force'' algorithm.

\begin{theorem}[\cite{M20}] \label{thm:label-cover-hardness}
Assuming Gap-ETH, for any function $f$ and any constants $\Delta \in \N \setminus \{1\}, \mu \in (0, 1)$, there is no $f(k) \cdot n^{o(k)}$-time algorithm that can, given a decomposable Label Cover instance $\cL = (U, V = V_1 \cup \cdots \cup V_t, E, \Sigma_U, \Sigma_V, \{\pi_e\}_{e \in E})$ whose right-degree is equal to $\Delta$, distinguish between
\begin{itemize}
\item (Completeness) $\val(\cL) = 1$,
\item (Soundness) $\wval(\cL) < \mu$,
\end{itemize}
where $k := |U|$ and $n := |U| \cdot |\Sigma_U| + |V| \cdot |\Sigma_V|$.
\end{theorem}

We remark here that the above theorem is not exactly the same as stated in~\cite{M20}. We now briefly explain how to derive the version above from the one in~\cite{M20}. Specifically, in~\cite{M20}, the decomposability of the instance $\cL$ is not stated; rather, the instance there has the following property: $V$ is simply all subsets of size $\Delta$ of $U$, and, for any vertex $\{u_1, \dots, u_{\Delta}\} \in V$, its neighbors are $u_1, \dots, u_{\Delta} \in U$. Now, we can assume w.l.o.g. that $k$ is divisible by $\Delta$ by expanding each vertex $u \in U$ to $\Delta$ new vertices $u^1, \dots, u^\Delta$ and replicate each vertex in $\{u_1, \dots, u_{\Delta}\} \in V$ to $\Delta^{\Delta}$ new vertices $\{u_1^{\xi(1)}, \dots, u_{\Delta}^{\xi(\Delta)}\}$ for all $\xi: [\Delta] \to [\Delta]$. Once we have that $k$ is divisible by $\Delta$, Baranyai's theorem~\cite{Baranyai75} immediately implies the decomposability of the instance.

\subsection{Anti-Concentration}

It is well-known that, if we take $m$ i.i.d. Rademacher random variables, their sum divided by $\sqrt{m}$ converges in distribution to the standard normal distribution (see, e.g.,~\cite{Berry41,Esseen42}). As a consequence, this immediately implies the following ``anti-concentration'' style result:
\begin{lemma} \label{lem:anti-concen}
There exists $C \in (0, 1)$ and $m_0 > 0$ such that, for any $m \geq m_0$, we have
\begin{align*}
\Pr_{X_1, \dots, X_m}[X_1 + \cdots + X_m \geq C \sqrt{m}] \geq 0.4 \;,
\end{align*}
where $X_1, \dots, X_m$ are i.i.d. Rademacher random variables.
\end{lemma}

Note that the constant 0.4 above can be replaced by any constant strictly less than 0.5. We only use 0.4 here to avoid introducing additional variables.

\section{Tight Running Time Lower Bound}
\label{sec:hardness}
Given the background from Section~\ref{sec:prelim}, in this section we proceed
to prove our computational lower bound (Theorem~\ref{thm:running-time-lower-bound}). As alluded to in the previous section, the main ingredient of our hardness result is a reduction from Label Cover to the problem of $L_{\infty}$-$\gamma$-margin halfspace learning. The properties of our reduction are summarized below.

\begin{theorem}[Hardness Reduction] \label{thm:reduction}
There exist absolute constants $\Delta, k_0 \in \N \setminus \{1\}$ and $\mu, \delta > 0$ such that the following holds. There is a polynomial time reduction that takes in a decomposable Label Cover instance $\cL = (U, V = V_1 \cup \cdots \cup V_t, E, \Sigma_U, \Sigma_V, \{\pi_e\}_{e \in E})$ whose right-degree is equal to $\Delta$, and produces real numbers $\gamma^*, \eps^* > 0$ and an oracle $\cO$ that can draw a sample from a distribution $\cD$ on $\bB_{\infty}^{|U| \cdot |\Sigma_U| + 1} \times \{\pm 1\}$ in polynomial time, such that when $|U| \geq k_0$ we have:
\begin{itemize}
\item (Completeness) If $\cL$ is fully satisfiable (i.e., $\val(\cL) = 1$), then $\opt^{\cD}_{\gamma^*} \leq \eps^*$.
\item (Soundness) If $\wval(\cL) < \mu$, then $\opt^{\cD}_{(1 - \delta)\gamma^*} > 1.6 \eps^*$.
\item (Margin Bound) $\gamma^* \geq \Omega(1/\sqrt{k})$.
\item (Error Bound) $\eps^* \geq n^{-O(\sqrt{k})}$. 
\end{itemize}
Here $k := |U|$ and $n := |U| \cdot |\Sigma_U| + |V| \cdot |\Sigma_V|$ are defined similarly to Theorem~\ref{thm:label-cover-hardness}.
\end{theorem}

We remark that, similar to Theorem~\ref{thm:running-time-lower-bound}, the constant 1.6 in the soundness above can be changed to any constant strictly less than two. However, we choose to use an explicit constant here to avoid having a further variable.

Before we prove Theorem~\ref{thm:reduction}, let us briefly argue that it implies the desired running time lower bound (Theorem~\ref{thm:running-time-lower-bound}).

\begin{proof}[Proof of Theorem~\ref{thm:running-time-lower-bound}]
Let $\Delta, k_0, \mu, \delta$ be as in Theorem~\ref{thm:reduction}. Suppose for the sake of contradiction that there exists a proper $\delta$-robust 1.5-agnostic learner for $L_{\infty}$-$\gamma$-margin halfspace $\A$ that runs in time $f(1/\gamma) \cdot d^{o(1 / \gamma^2)} \poly(1/\eps)$. We will use this to construct an algorithm $\B$ for Label Cover.

Given a Label Cover instance $\cL$ as an input, the algorithm $\B$ works as follows:
\begin{itemize}
\item Run the reduction from Theorem~\ref{thm:reduction} on input $\cL$ to get $\eps^*, \gamma^*, \cO$.
\item Run $\A$ on $\cO$ with parameters $\gamma = \gamma^*, \epsilon = 0.05\eps^*, \tau = 0.9$ to get a halfspace $\bw$.
\item Draw $10^6 / \eps^2$ additional samples from $\cO$. Let $\tilde{\cD}$ be the empirical distribution.
\item If $\err^{\tilde{\cD}}_{(1 - \delta)\gamma}(\bw) \leq 1.58\eps^*$, return YES. Otherwise, return NO.
\end{itemize}
The first step of $\B$ runs in $\poly(n)$ time. The second step runs in time $f(1 / \gamma) d^{o(1 / \gamma^2)} \poly(1/\eps) = f(O(\sqrt{k})) \cdot n^{o(k)} \cdot \poly(n^{O(\sqrt{k})}) = f(O(\sqrt{k})) \cdot n^{o(k)}$. The last two steps run in time $\poly(n, 1/\eps)$; recall from Theorem~\ref{thm:reduction} that $\eps^* = n^{O(\sqrt{k})}$, meaning that these two steps run in time $n^{O(\sqrt{k})}$. Hence, the entire algorithm $\B$ runs in $g(k) \cdot n^{o(k)}$ time for some function $g$. 

We will next argue the following correctness guarantee of the algorithm: If $\val(\cL) = 1$, then the algorithm answers YES with probability $0.8$ and, if $\wval(\cL) < \mu$, then the algorithm returns NO with probability $0.8$. Before we do so, observe that this, together with Theorem~\ref{thm:label-cover-hardness}, means that Gap-ETH is violated, which would complete our proof.

Note that we may assume that $k \geq k_0$, as otherwise the Label Cover instance can already be solved in polynomial time. Now consider the case $\val(\cL) = 1$. Theorem~\ref{thm:reduction} ensures that $\opt^{\cD}_{\gamma^*} \leq \eps^*$. As a result, $\A$ returns $\bw$ that satisfies the following with probability 0.9: $\err^{\cD}_{(1 - \delta)\gamma^*}(\bw) \leq 1.5 \eps^* + 0.05\eps^* = 1.55\eps^*$. Furthermore, it is simple to check that $\Pr[|\err^{\cD}_{(1 - \delta)\gamma^*}(\bw) - \err^{\tilde{\cD}}_{(1 - \delta)\gamma^*}(\bw)| > 0.02\eps^*] \leq 0.1$. Hence, with probability 0.8, we must have $\err^{\tilde{\cD}}_{(1 - \delta)\gamma^*} \leq 1.57\eps^*$ and the algorithm returns YES.

On the other hand, suppose that $\wval(\cL) < \mu$. The soundness of Theorem~\ref{thm:reduction} ensures that $\err^{\cD}_{(1 - \delta)\gamma^*}(\bw) > 1.6 \eps^*$. Similar to before, since $\Pr[|\err^{\cD}_{(1 - \delta)\gamma^*}(\bw) - \err^{\tilde{\cD}}_{(1 - \delta)\gamma^*}(\bw)| > 0.02\eps^*] \leq 0.1$, we have $\err^{\tilde{\cD}}_{(1 - \delta)\gamma^*}(\bw) > 1.58 \eps^*$ with probability at least 0.9. Thus, in this case, the algorithm returns NO with probability 0.9 as desired. 
\end{proof}

The remainder of this section is devoted to the proof of Theorem~\ref{thm:reduction} and is organized as follows. First, in Section~\ref{sec:informal-overview}, we give an informal overview of techniques and compare our reduction to those of previous works. Then, in Section~\ref{subsec:reduction}, we give a formal description of our reduction together with the choice of parameters. The completeness and soundness of the reduction are then proved in Sections~\ref{subsec:completeness} and~\ref{subsec:soundness} respectively. 

\subsection{Overview of Techniques}
\label{sec:informal-overview}

We will now give a high-level overview of the reduction. For simplicity of presentation, we will sometimes be informal; everything will be formalized in the subsequent subsections.

\paragraph{Previous Results.} 
To explain the key new ideas behind our reduction, 
it is important to understand high-level approaches taken in previous works 
and why they fail to yield running time lower bounds 
as in our Theorem~\ref{thm:running-time-lower-bound}.

Most of the known hardness results 
for agnostic learning of halfspaces employ reductions 
from Label Cover~\cite{AroraBSS97,FeldmanGKP06,GuruswamiR09,FeldmanGRW12,DiakonikolasKM19}\footnote{Some of these reductions are stated in terms of reductions from Set Cover or from constraint satisfaction problems (CSP). However, it is well-known that these can be formulated as Label Cover.}. 
These reductions use gadgets which are ``local'' in nature. 
As we will explain next, such ``local'' reductions \emph{cannot} work for our purpose.

To describe the reductions, it
is convenient to think of each sample $(\bx, y)$ 
as a linear constraint $\left<\bw, \bx\right> \geq 0$ when $y = +1$ 
and $\left<\bw, \bx\right> < 0$ when $y = -1$, 
where the variables are the coordinates $w_1, \dots, w_d$ of $\bw$. 
When we also consider a margin parameter $\gamma^* > 0$, then the constraints become 
$\left<\bw, \bx\right> \geq \gamma^*$ and $\left<\bw, \bx\right> < -\gamma^*$, respectively. Notice here that, for our purpose, we want 
(i) our halfspace $\bw$ to be in $\bB^d_{1}$, i.e., $|w_1| + \cdots + |w_d| \leq 1$, 
and (ii) each of our samples $\bx$ to lie in $\bB^d_{\infty}$, i.e., $|x_1|, \dots, |x_d| \leq 1$.

Although the reductions in previous works vary in certain steps, 
they do share an overall common framework. With some
simplification, 
they typically let e.g. $d = |U| \cdot |\Sigma_U|$, 
where each coordinate is associated with $U \times \Sigma_U$. 
In the completeness case, i.e., when some labeling $\phi^c$ covers all vertices in $V$, 
the intended solution $\bw^c$ is
defined by $w^c_{(u, \sigma_u)} = \1[\sigma_u = \phi(u)] / k$ for all $u \in U, \sigma_u \in \Sigma_U$. To ensure that this is essentially the best choice of halfspace, 
these reductions often appeal to several types of linear constraints. 
For concreteness, we state a simplified version of those from~\cite{AroraBSS97} below.
\begin{itemize}
\item For every $(u, \sigma_U) \in U \times \Sigma_U$, 
create the constraint $w_{(u, \sigma_u)} \leq 0$.
(This corresponds to the labeled sample $(-\be_{(a, \sigma)}, +1)$.)

\item For each $u \in U$, create the constraint 
$\sum_{\sigma \in \Sigma_U} w_{(u, \sigma)} \geq 1/k$. 
\item For every $v \in V$, $\sigma_v \in \Sigma_V$ and $u_1, u_2 \in N(v)$, 
add $\sum_{\sigma_{u_1} \in \pi_{(u_1, v)}^{-1}(\sigma_v)} w_{(u_1, \sigma_{u_1})} = \sum_{\sigma_{u_2} \in \pi_{(u_2, v)}^{-1}(\sigma_v)} w_{({u_2}, \sigma_{u_2})}$.
This equality ``checks'' the Label Cover constraints $\pi_{(u_1, v)}$ and $\pi_{(u_2, v)}$.
\end{itemize}
Clearly, in the completeness case $\bw^c$ satisfies all constraints except the non-positivity constraints 
for the $k$ non-zero coordinates.
(It was argued in~\cite{AroraBSS97} that any halfspace must violate many more constraints in the soundness case.)
%
Observe that this reduction does not yield any margin: $\bw^c$ does \emph{not} classify any sample with a positive margin. 
Nonetheless,~\cite{DiakonikolasKM19} adapts this reduction to work with a small margin $\gamma^* > 0$ 
by adding/subtracting appropriate ``slack'' from each constraint. 
For example, the first type of constraint is changed to 
$w_{(u, \sigma_u)} \leq \gamma^*$.
This gives the desired margin $\gamma^*$ in the completeness case. 
However, for the soundness analysis to work, it is crucial that 
$\gamma^* \leq O(1/k)$, 
as otherwise the constraints can be trivially satisfied\footnote{Note that $\bw = \bzero$ satisfies 
the constraints with margin $\gamma^* - 1/k$, which is 
$(1 - o(1))\gamma^*$ if $\gamma^* = \omega(1/k)$.} by $\bw = \bzero$. 
As such, the above reduction does \emph{not} work for us, 
since we would like a margin $\gamma^* = \Omega(1/\sqrt{k})$. 
In fact, this also holds for all known reductions, which are ``local'' in nature and possess 
similar characteristics. Roughly speaking, each linear constraint of these reductions 
involves only a constant number of terms that are intended to be set to $O(1/k)$, 
which means that we cannot hope to get a margin more than $O(1/k)$.

\paragraph{Our Approach: Beyond Local Reductions.} 
With the preceding discussion in mind, our reduction has to be ``non-local'', i.e., 
each linear constraint has to involve many of the non-zero coordinates. 
Specifically, for each subset $V^j$, we will check all the Label Cover constraints 
involving $v \in V^j$ at once. To formalize this goal, we will require 
the following definition.
\begin{definition} \label{def:projection-matrix}
Let $\cL = (U, V = V_1 \cup \cdots \cup V_t, E, \Sigma_U, \Sigma_V, \{\pi_e\}_{e \in E})$ be a decomposable Label Cover instance. 
For any $j \in [t]$, let $\Pi^j \in \mathbb{R}^{(V \times \Sigma_V) \times (U \times \Sigma_U)}$ be defined as
\begin{align*}
\Pi^j_{(v, \sigma_v), (u, \sigma_u)} =
\begin{cases}
1 & \text{ if } v = v^j(u) \text{ and } \pi_{(u, v)}(\sigma_u) = \sigma_v, \\
0 & \text{ otherwise}.
\end{cases}
\end{align*}
\end{definition}

We set $d = |U| \cdot |\Sigma_U|$ and our intended solution $\bw^c$ in the completeness case is the same as described in the previous reduction. 
For simplicity, suppose that, in the soundness case, 
we pick $\phi^s$ that does \emph{not} weakly cover any $v \in V$ and set 
$w^s_{(u, \sigma_u)} = \bone[\sigma_u = \phi^s(u)]/k$. 
Our simplified task then becomes:
\emph{Design $\cD$ such that $\err_{\gamma}^{\cD}(\bw^c) \ll \err_{(1 - \nu)\gamma}^{\cD}(\bw^s)$, 
where $\gamma = \Omega(1/\sqrt{k})$, $\nu > 0$ is a constant.} 

Our choice of $\cD$ is based on two observations. 
The first is a structural difference between $\bw^c (\Pi^j)^T$ and $\bw^s (\Pi^j)^T$. 
Suppose that the constraint graph has right degree $\Delta$. 
Since $\phi^c$ covers all $v \in V$, $\Pi^j$ ``projects'' 
the non-zeros coordinates $w^c_{(u, \phi^c(u))}$ for all $u \in N(v)$ 
to the same coordinate $(v, \sigma_v)$, for some $\sigma_v \in \Sigma_V$, 
resulting in the value of $\Delta / k$ in this coordinate. On the other hand, 
since $\phi^s$ does not even weakly cover any right vertex, 
all the non-zero coordinates get maps by $\Pi^j$ to different coordinates, 
resulting in the vector $\bw^s (\Pi^j)^T$ having $k$ non-zero coordinates, 
each having value $1/k$.

To summarize, we have:
$\bw^c (\Pi^j)^T$ has $k / \Delta$ non-zero coordinates, each of value $\Delta/k$. 
On the other hand, $\bw^s (\Pi^j)^T$ has $k$ non-zero coordinates, each of value $1/k$.

Our second observation is the following: 
suppose that $\bu$ is a vector with $T$ non-zero coordinates, 
each of value $1/T$. If we take a random $\pm 1$ vector $\bs$, then 
$\left<\bu, \bs\right>$ is simply $1/T$ times a sum of $T$ i.i.d. Rademacher random variables. 
Recall a well-known version of the central limit theorem (e.g.,~\cite{Berry41,Esseen42}): 
as $T \to \infty$, $1/\sqrt{T}$ times a sum of $T$ i.i.d. Rademacher r.v.s converges in distribution to the normal distribution. 
This implies that $\lim_{T \to \infty} \Pr[\left<\bu, \bs\right> \geq 1/\sqrt{T}] = \Phi(1).$

For simplicity, let us ignore the limit for the moment and assume that $\Pr[\left<\bu, \bs\right> \geq 1/\sqrt{T}] = \Phi(1)$. We can now specify 
the desired distribution $\cD$: Pick $\bs$ uniformly at random from $\{\pm 1\}^{V \times \Sigma_V}$ 
and then let the sample be $\bs\Pi^j$ with label $+1$. By the above two observations, 
$\bw^c$ will be correctly classified with margin 
$\gamma^* = \sqrt{\Delta / k} = \Omega(1/\sqrt{k})$ with probability $\Phi(1)$. 
Furthermore, in the soundness case, $\bw^s$ can only get the same error 
with margin (roughly) $\sqrt{1/k} = \gamma^* / \sqrt{\Delta}$. 
Intuitively, for $\Delta > 1$, this means that we get a gap of $\Omega(1/\sqrt{k})$ 
in the margins between the two cases, as desired. This concludes our informal proof overview.

\subsection{The Reduction}
\label{subsec:reduction}

Having stated the rough main ideas above, we next formalize the reduction. To facilitate this, we define the following additional notations:

\begin{definition}
Let $\cL = (U, V = V_1 \cup \cdots \cup V_t, E, \Sigma_U, \Sigma_V, \{\pi_e\}_{e \in E})$ be a decomposable Label Cover instance. 
For any $j \in [t]$, let $\hPi^j \in \mathbb{R}^{(U \times \Sigma_V) \times (U \times \Sigma_U)}$ be such that 
\begin{align*}
\hPi^j_{(u', \sigma_v), (u, \sigma_u)} = 
\begin{cases}
1 & \text{ if } u' = u \text{ and } \pi_{(u, v^j(u))}(\sigma_u) = \sigma_v, \\
0 & \text{ otherwise}.
\end{cases}
\end{align*}
Moreover, let $\tPi^j \in \mathbb{R}^{(V \times \Sigma_V) \times (U \times \Sigma_V)}$ be such that 
\begin{align*}
\tPi^j_{(v, \sigma'_v), (u, \sigma_v)} = 
\begin{cases}
1 & \text{ if } v = v^j(u) \text{ and } \sigma'_v = \sigma_v \\
0 & \text{ otherwise}.
\end{cases}
\end{align*}
 Observe that $\Pi^j = \tPi^j \cdot \hPi^j$ (where $\Pi^j$ is as in Definition~\ref{def:projection-matrix}).
\end{definition}

Our full reduction is present in Figure~\ref{fig:reduction} below. Before we specify the choice of parameters, let us make a few remarks. First, we note that the distribution described in the previous section corresponds to Step~\ref{sample:label-cover} in the reduction. The other steps of the reductions are included to handle certain technical details we had glossed over previously. In particular, the following are the two main additional technical issues we have to deal with here. 
\begin{itemize}
\item \emph{(Non-Uniformity of Weights)} In the intuitive argument above, we assume that, 
in the soundness case, we only consider $\bw^s$ such that 
$\sum_{\sigma_u \in \Sigma_U} w_{(u, \sigma_u)}^s = 1/k$. However, 
this need not be true in general, and we have to create new samples to (approximately) enforce 
such a condition. Specifically, for every subset $T \subseteq U$, we add a constraint 
that $\sum_{u \in T} \sum_{\sigma_u \in \Sigma_U} w_{(u, \sigma_u)} \geq |T|/k - \gamma^*$. 
This corresponds to Step~\ref{sample:mass-lower} in Figure~\ref{fig:reduction}.

Note that the term $- \gamma^*$ on the right hand side above is necessary to ensure that, 
in the completeness case, we still have a margin of $\gamma^*$. 
Unfortunately, this also leaves the possibility of, e.g., some vertex $u \in U$ has 
as much as $\gamma^*$ extra ``mass''. For technical reasons, it turns out that 
we have to make sure that these extra ``masses'' do not contribute to too much of $\|\bw(\Pi^j)^T\|_2^2$. 
To do so, we add additional constraints on $\bw(\hPi^j)^T$ to bound its norm. 
Such a constraint is of the form: If we pick a subset $S$ of at most $\ell$ coordinates, 
then their sum must be at most $|S|/k + \gamma^*$ (and at least $-\gamma^*$). 
These corresponds to Steps~\ref{sample:exceed-l2-norm} and~\ref{sample:neg-l2-norm} 
in Figure~\ref{fig:reduction}.

\item \emph{(Constant Coordinate)} Finally, similar to previous works, 
we cannot have ``constants'' in our linear constraints. Rather, we need to add 
a coordinate $\star$ with the intention that $\bw_{\star} = 1/2$, 
and replace the constants in the previous step by $\bw_{\star}$. 
Note here that we need two additional constraints (Steps~\ref{sample:constant-lower} 
and~\ref{sample:constant-upper} in Figure~\ref{fig:reduction}) to ensure that 
$\bw_{\star}$ has to be roughly $1/2$.
\end{itemize} 

\begin{figure}[h!]
\begin{framed}

\textbf{Input:} Decomposable Label Cover instance $\cL = (U, V = V_1 \cup \dots \cup V_t, E, \Sigma_U, \Sigma_V, \{\pi_e\}_{e \in E})$.

\textbf{Parameters:} $q, \gamma^* \in (0, 1), \ell \in \N$. 

\textbf{Output:} Oracle $\cO$ that draws a sample from a distribution $\cD$ on $\bB_{\infty}^{|U| \cdot |\Sigma_U| + 1} \times \{\pm 1\}$. \\

For notational convenience, we associate each coordinate of $(|U|\cdot|\Sigma_U| + 1)$-dimensional samples with an element from $(U \times \Sigma_U) \cup \{\star\}$. The oracle $\cO$ draws a sample as follows:
\begin{enumerate}
\item With probability $0.25$, output the sample $2\gamma^* \cdot \be_{\star}$ with label +1. \label{sample:constant-lower}
\item With probability $0.25$, output the sample $2\gamma^* \cdot \be_{U \times \Sigma_U}$ with label +1. \label{sample:constant-upper}
\item With probability $0.25$, pick a random subset $T \subseteq U$ and output the sample $\be_{T \times \Sigma_U} - \left(\frac{|T|}{k} - 2\gamma^*\right) \be_{\star}$ with label +1. \label{sample:mass-lower}
\item With probability $0.25$, draw $j$ uniformly at random from $[t]$. Then, do the following:
\begin{enumerate}
\item With probability $0.5(1 - q)$, randomly pick a subset $S \subseteq U \times \Sigma_V$ of size at most $\ell$. Output the labeled sample $((\frac{|S|}{k} + 2\gamma^*) \be_{\star} - \be_S \hPi^j, +1)$. \label{sample:exceed-l2-norm}
\item With probability $0.5(1 - q)$, randomly pick a subset $S \subseteq U \times \Sigma_V$ of size at most $\ell$. Then, output $(2\gamma^* \be_{\star} + \be_S \hPi^j, +1)$. \label{sample:neg-l2-norm}
\item With probability $q$, sample $\bs$ uniformly at random from $\{\pm 1\}^{V \times \Sigma_V}$ and, output $(\bs \Pi^j, +1)$. \label{sample:label-cover}
\end{enumerate}
\end{enumerate}
\end{framed}
\caption{Hardness Reduction from Label Cover to $L_{\infty}$-margin Halfspace Learning.} \label{fig:reduction}
\end{figure}

The parameters of our reduction are set as follows:
\begin{itemize}
\item $C$ and $m_0$ are as in Lemma~\ref{lem:anti-concen},
\item $\Delta = \lceil 10^4 / C^2 \rceil$,
\item $\gamma^* = 0.5 C \sqrt{\Delta/k}$,
\item $k_0 = m_0 \Delta$,
\item $\delta = (0.1 / \Delta)^4$,
\item $\ell = \lceil\delta\sqrt{k}\rceil$,
\item $q = 0.001 / n^{\ell}$ (where $n$ is as defined is Theorem~\ref{thm:label-cover-hardness}),
\item $\eps^* = 0.6 (0.25 q)$,
\item $\mu = \frac{0.01}{\Delta(\Delta - 1)}$.
\end{itemize}

It is easy to see that the oracle can draw a sample in polynomial time. Furthermore, $\eps^* = 0.001 / n^{\ell} \geq n^{-O(\sqrt{k})}$ and $\gamma^* \geq \Omega(1/\sqrt{k})$, as desired. Hence, we are only left to prove the completeness and the soundness of the reduction, which we will do next.

\subsection{Completeness}
\label{subsec:completeness}

Suppose that the Label Cover instance $\cL$ is satisfiable, i.e., that there exists a labeling $\phi^*$ that covers all right vertices. Let $\bw^*$ be such that $w^{*}_{\star} = 1/2$ and
\begin{align*}
w^*_{(u, \sigma)} =
\begin{cases}
\frac{1}{2k} & \text{ if } \sigma = \phi^*(u), \\
0 & \text{ otherwise}
\end{cases}
\end{align*}
for all $u \in U, \sigma \in \Sigma_U$. It is simple to check that the samples generated in Steps~\ref{sample:constant-lower},~\ref{sample:constant-upper},~\ref{sample:mass-lower},~\ref{sample:exceed-l2-norm} and~\ref{sample:neg-l2-norm} are all correctly labeled with margin $\gamma^*$.

Hence, we are left with computing the probability that the samples generated in Step~\ref{sample:label-cover} are violated. To do this, first notice that, for every $j \in [t], v \in V^j, \sigma_v \in \Sigma_V$, we have
\begin{align*}
(\bw^* (\Pi^j)^T)_{(v, \sigma_v)} &= \sum_{(u, \sigma_u) \in U \times \Sigma_U \atop v^j(u) = v, \pi_{(u, v)}(\sigma_u) = \sigma_v} w^*_{(u, \sigma_u)} \\
(\text{From definition of } w^*) &= \frac{1}{2k} \left|\{u \in N(v) \mid \pi_{(u, v)}(\phi^*(u)) = \sigma_v\}\right|.
\end{align*}
Now since every $v \in V^j$ is covered by $\phi^*$, there exists a unique $\sigma_v$ such that $\pi_{(u, v)}(\phi^*(u)) = \sigma_v$ for all $u \in N(v)$. As a result, $\bw^* (\Pi^j)^T$ has $|V^j| = k/\Delta$ coordinates exactly equal to $\Delta \cdot \frac{1}{2k} = \frac{\Delta}{2k}$, and the remaining coordinates are equal to zero. Recall that, for the samples in Step~\ref{sample:label-cover}, $\bs$ is a random $\{\pm 1\}$ vector. Thus, $\left<\bw^*, \bs \Pi^j\right> = \left<\bw^*(\Pi^j)^T, \bs\right>$ has the same distribution as $\frac{\Delta}{2k}$ times a sum of $k/\Delta$ i.i.d. Rademacher random variables. By Lemma~\ref{lem:anti-concen}, we can conclude that $\Pr_{\bs}[\left<\bw^*, \bs \Pi^j\right> \geq 0.5 C\sqrt{\Delta/k}] \geq 0.4$. Since we set $\gamma^* = 0.5 C\sqrt{\Delta/k}$, this implies that $\bw^*$ correctly classifies (at least) $0.4$ fraction of the samples from Step~\ref{sample:label-cover}. Hence, we have
\begin{align*}
\err^{\cD}_{\gamma}(\bw^*) \leq 0.6 \cdot (0.25q) = \eps^* \;,
\end{align*}
as desired.

\subsection{Soundness}
\label{subsec:soundness}

We will prove the soundness contrapositively. For this purpose, suppose that there is a halfspace $\bw \in \bB_1^d$ such that $\err^{\cD}_{(1 - \delta)\gamma}(\bw) \leq 1.6 \eps^* = 0.96 (0.25q)$. We will show that there exists an assignments $\phi'$ with $\wval(\phi') \geq \mu$.

\subsubsection{Some Simple Bounds}

We start by proving a few observations/lemmas that will be useful in the subsequent steps.

First, observe that every distinct sample from Steps~\ref{sample:constant-lower},~\ref{sample:constant-upper},~\ref{sample:mass-lower},~\ref{sample:exceed-l2-norm} and~\ref{sample:neg-l2-norm} has probability mass (in $\cD$) at least $\frac{0.125(1 - q)}{n^{\ell}} > q > 1.6 \eps^*$. Since we assume that $\err^{\cD}_{(1 - \delta)\gamma}(\bw) \leq 1.6 \eps^*$, it must be the case that all these examples are correctly classified by $\bw$ with margin at least $(1 - \delta)\gamma^*$:

\begin{observation} \label{obs:correctly-classify}
$\bw$ correctly classifies all samples in Steps~\ref{sample:constant-lower},~\ref{sample:constant-upper},~\ref{sample:mass-lower},~\ref{sample:exceed-l2-norm} and~\ref{sample:neg-l2-norm} with margin $(1 - \delta)\gamma^*$.
\end{observation}

Throughout the remainder of this section, we will use the following notations:
\begin{definition}
For every $u \in U$, let $M_u$ denote $\sum_{\sigma \in \Sigma_u} |w_{(u, \sigma)}|$. Then, let $U_{\text{small}}$ denote $\{u \in U \mid M_u \leq 1/k\}$ and $U_{\text{large}}$ denote $U \setminus U_{\text{small}}$.
\end{definition}

The next observation, which follows almost immediately from Observation~\ref{obs:correctly-classify}, is that the value of the ``constant coordinate'' $w_{\star}$ is roughly $1/2$ (as we had in the completeness case) and that the sum of the absolute values of the negative coordinates is quite small.

\begin{observation} \label{obs:constant-and-non-neg}
The following holds:
\begin{enumerate}
\item (Constant Coordinate Value) $w_{\star} \in [0.5(1 - \delta), 0.5(1 + \delta)]$.
\item (Negative Coordinate Value) $\sum_{j \in (U \times \Sigma_U) \cup \{\star\} \atop w_j < 0} |w_j| \leq \delta$.
\end{enumerate}
\end{observation}

\begin{proof}
\begin{enumerate}
\item Since $\bw$ correctly classifies the sample from Step~\ref{sample:constant-lower} with margin $(1 - \delta)\gamma^*$, we have $2\gamma^* w_{\star} > (1 - \delta)\gamma^*$. This implies that $w_{\star} \geq 0.5(1 - \delta)$. 

Let $a = \left<\bw, \be_{U \times \Sigma_U}\right>$. Similarly, from $\bw$ correctly classifies the sample from Step~\ref{sample:constant-upper} with margin $(1 - \delta)\gamma^*$, we have $a \geq 0.5(1 - \delta)$. Furthermore, observe that
\begin{align} \label{eq:norm-ineq}
a + w_{\star} \leq \|\bw\|_1 \leq 1.
\end{align}
As a result, we have $w_{\star} \leq 0.5(1 + \delta)$ as desired.
\item Since $w_{\star} > 0$, we may rearrange the desired term as
\begin{align*}
\sum_{j \in (U \times \Sigma_U) \cup \{\star\} \atop w_j < 0} |w_j| &= \frac{1}{2} \left(\|\bw\|_1 - a - w_{\star}\right) \\
&\leq \frac{1}{2}\left(1 - 0.5(1 - \delta) - 0.5(1 - \delta)\right) \\
&< \delta,
\end{align*}
where the first inequality follows from $a, w^* \geq 0.5(1 - \delta)$ that we had shown above. \qedhere
\end{enumerate}
\end{proof}

Another bound we will use is that $U_{\text{large}}$ is quite small, and the sum of absolute values of the coordinates correspond to $U_{\text{large}}$ is also quite small.

\begin{observation}[Bounds on $U_{\text{large}}$] \label{obs:large-mass}
The following holds:
\begin{enumerate}
\item (Size Bound) $|U_{\text{large}}| \leq 2\delta k$.
\item (Mass Bound) $\sum_{u \in U_{\text{large}}} M_u \leq 2\delta$.
\end{enumerate}
\end{observation}

\begin{proof}
To prove the desired bounds, first notice that, since $\bw$ correctly classifies the sample in Step~\ref{sample:mass-lower} with $T = U_{\text{small}}$ with margin $(1 - \delta) \gamma^*$, we must have
\begin{align*}
\left<\bw, \be_{U_{\text{small}} \times \Sigma_U}\right> \geq \left(\frac{|U_{\text{small}}|}{k} - 2\gamma^*\right) w_{\star} + (1 - \delta) \gamma^*.
\end{align*}
Now, observe that the term on the left hand side is at most $\sum_{u \in U_{\text{small}}} M_u$ which, from $\|\bw\|_1 \leq 1$, is in turn at most $1 - w_{\star} - \sum_{u \in U_{\text{large}}} M_u$. Combining these, we get
\begin{align*}
1 - w_{\star} - \sum_{u \in U_{\text{large}}} M_u \geq \left(\frac{|U_{\text{small}}|}{k} - 2\gamma^*\right) w_{\star} + (1 - \delta) \gamma^* = \left(1 - \frac{|U_{\text{large}}|}{k} - 2\gamma^*\right) w_{\star} + (1 - \delta) \gamma^*
\end{align*}
Recall from Observation~\ref{obs:constant-and-non-neg} that $w_{\star} \geq 0.5(1 - \delta)$. Plugging this into the above, we have
\begin{align}
\sum_{u \in U_{\text{large}}} M_u 
&\leq 1 - \left(2 - \frac{|U_{\text{large}}|}{k} - 2\gamma^*\right) \cdot 0.5(1 - \delta) - (1 - \delta)\gamma^* \nonumber \\
&= 1 - \left(2 - \frac{|U_{\text{large}}|}{k}\right) \cdot 0.5(1 - \delta) \nonumber \\
&\leq \delta + \frac{0.5|U_{\text{large}}|}{k} \;. \label{eq:mass-bound-intermediate}
\end{align}
\begin{enumerate}
\item Subtracting $\frac{0.5 |U_{\text{large}}|}{k}$ from both sides, we have
\begin{align*}
\sum_{u \in U_{\text{large}}} \left(M_u - \frac{0.5}{k}\right) \leq \delta.
\end{align*}
By definition, $M_u > 1/k$ for all $u \in U_{\text{large}}$. As a result, we have $|U_{\text{large}}| \leq 2\delta k$, as desired.
\item Plugging the bound on $|U_{\text{large}}|$ back into~\eqref{eq:mass-bound-intermediate}, we get the claimed bound on $\sum_{u \in U_{\text{large}}} M_u$. \qedhere
\end{enumerate}
\end{proof}

\subsubsection{Identifying a ``Nice'' Halfspace}

We will now convert $\bw$ into a ``nicer'' halfspace, i.e., one without negative and large coordinates. It will be much more convenient to deal with such a nice halfspace when we ``decode'' back a labeling later in this section.

The ``nice'' halfspace is quite simple: we just zero out all coordinates $w_{(u, \sigma)}$, where $u \in U_{\text{large}}$. More formally, let $\bhw \in \R^{|U| \cdot |\Sigma_U|}$ be such that
\begin{align*}
\hw_{(u, \sigma)} =
\begin{cases}
w_{(u, \sigma)} & u \in U_{\text{small}},  \\
0 & u \in U_{\text{large}},
\end{cases}
\end{align*}
for all $u \in U$ and $\sigma \in \Sigma_U$.




The main lemma needed in our analysis is that, for each $j \in [t]$, $\bhw (\Pi^j)^T$ preserves most of the $L_2$ norm compared to the original $\bw (\Pi^j)^T$.

\begin{lemma}[Nice Halfspace Preserves Most of $L_2$ Norm] \label{lem:variance-preserved}
For every $j \in [t]$, we have 
\begin{align} \label{eq:variance-preserved}
\|\bhw(\Pi^j)^T\|_2^2 \geq \frac{\|\bw (\Pi^j)^T\|_2^2}{2} - \frac{\sqrt[4]{\delta}}{k}.
\end{align}
\end{lemma}

\begin{proof}
For convenience, let $\bd = \bw - \bhw$ and $\bb = \bv(\hPi^j)^T$. The majority of this proof is spent on bounding $\|\bb\|_2^2$. To do this, let us define several new notations:
\begin{itemize}
\item Let $C_{> 0} = |\{(u, \sigma) \in U \times \Sigma_V \mid b_{(u, \sigma)} > 0\}|$ and $C_{< 0} = |\{(u, \sigma) \in U \times \Sigma_V \mid b_{(u, \sigma)} < 0\}|$.
\item Let $\bb^{\geq 0} \in \R^{U \times \Sigma_V}$ be defined by
\begin{align*}
b^{\geq 0}_{(u, \sigma)} = \max\{0, b_{(u, \sigma)}\}
\end{align*} for all $(u, \sigma) \in U \times \Sigma_V$. Furthermore, let $\bb^{< 0} = \bb - \bb^{\geq 0}$.
\end{itemize}

Observe that $\|\hPi^j\|_1 \leq 1$, because each column has exactly a single entry equal to one and the remaining entries equal to zero. As a result, we have
\begin{align} \label{eq:l1-norm-b}
\|\bb\|_1 = \|\bv(\hPi^j)^T\|_1 \leq \|\bd\|_1 = \sum_{u \in U_{large}} M_u \leq 2\delta \;,
\end{align}
where the last inequality follows from Observation~\ref{obs:large-mass}.

Since $\bb = \bb^{\geq 0} + \bb^{< 0}$, we may bound $\|\bb^{\geq 0}\|_2, \|\bb^{< 0}\|_2$ separately, starting with the former.

\paragraph{Bounding $\|\bb^{\geq 0}\|_2$.} Let us sort the coordinates of $\bb^{\geq 0}$ from largest to smallest entries as $b^{\geq 0}_{(u^1, \sigma^1)}, \dots,$ $b^{\geq 0}_{(u^{|U| \times |\Sigma_V|}, \sigma^{|U| \times |\Sigma_V|})}$ (tie broken arbitrarily). For every $j \leq \min\{C_{> 0}, \ell\}$, consider the sample from Step~\ref{sample:exceed-l2-norm} when $S = \{(u^1, \sigma^1), \dots, (u^j, \sigma^j)\}$. Since $\bw$ correctly classifies this sample with margin $(1 - \delta)\gamma^*$, we have
\begin{align*}
(1 - \delta)\gamma^* &\leq \left<\bw, \left(\frac{j}{k} + 2\gamma^*\right) \be_{\star} - \be_S \hPi^j\right> \\
&= \left(\frac{j}{k} + 2\gamma^*\right) w_{\star} - \bw(\hPi^j)^T(\be_S)^T \\
&= \left(\frac{j}{k} + 2\gamma^*\right) w_{\star} - \left(\sum_{i \in [j]} (\bw(\hPi^j)^T)_{(u^i, \sigma^i)}\right) \\
\text{(Observation~\ref{obs:constant-and-non-neg})} &\leq \left(\frac{j}{k} + 2\gamma^*\right)\cdot 0.5(1 + \delta) - \left(\sum_{i \in [j]} (\bw(\hPi^j)^T)_{(u^i, \sigma^i)}\right) \\
&= \left(\frac{j}{k} + 2\gamma^*\right)\cdot 0.5(1 + \delta) - \left(\sum_{i \in [j]} \left(b^{\geq 0}_{(u^i, \sigma^i)} + (\bhw(\hPi^j)^T)_{(u^i, \sigma^i)}\right)\right) \\
&= \left(\frac{j}{k} + 2\gamma^*\right)\cdot 0.5(1 + \delta) - \left(\sum_{i \in [j]} b^{\geq 0}_{(u^i, \sigma^i)}\right),
\end{align*}
where the last equality follows from the fact that, for every $i \leq C_{> 0}$, we must have $u^i \in U_{\text{large}}$ as otherwise $b^{\geq 0}_{(u^i, \sigma^i)}$ would have been equal to zero.

Rearranging the above inequality, we have
\begin{align*}
\left(\sum_{i \in [j]} b^{\geq 0}_{(u^i, \sigma^i)}\right) \leq \frac{0.5(1 + \delta)j}{k} + 2\delta\gamma^* \leq \frac{j}{k} + 2\delta\gamma^*.
\end{align*}
Recall from our assumption that $b^{\geq 0}_{(u^1, \sigma^1)} \geq \cdots \geq b^{\geq 0}_{(u^j, \sigma^j)}$. Plugging this into the above, we get
\begin{align} \label{eq:term-by-term-bound-positive}
b^{\geq 0}_{(u^j, \sigma^j)} \leq \frac{1}{k} + \frac{2\delta \gamma^*}{j}.
\end{align}
Notice that while we have only derived the above inequality for $j \leq \min\{C_{> 0}, \ell\}$, it also extends to all $j \leq \ell$ because $b^{\geq 0}_{(u^j, \sigma^j)} = 0$ for all $j > C_{> 0}$.

We can use this to bound $\|\bb^{\geq 0}\|_2^2$ as follows.
\begin{align*}
\|\bb^{\geq 0}\|_2^2 &= \sum_{j=1}^{|U| \cdot |\Sigma_V|} \left(b^{\geq 0}_{(u^j, \sigma^j)}\right)^2 \\
&= \sum_{j < \ell} \left(b^{\geq 0}_{(u^j, \sigma^j)}\right)^2 + \sum_{j \geq \ell} \left(b^{\geq 0}_{(u^j, \sigma^j)}\right)^2 \\
&\leq \sum_{j < \ell} \left(b^{\geq 0}_{(u^j, \sigma^j)}\right)^2 + b^{\geq 0}_{(u^\ell, \sigma^\ell)} \cdot \|\bb^{\geq 0}\|_1 \\
&\overset{\eqref{eq:term-by-term-bound-positive}}{\leq} \sum_{j < \ell} \left(\frac{1}{k} + \frac{2\delta\gamma^*}{j}\right)^2 + \left(\frac{1}{k} + \frac{2\delta \gamma^*}{\ell}\right) \cdot \|\bb\|_1 \\
&\overset{\eqref{eq:l1-norm-b}}{\leq} \sum_{j < \ell} 2\left(\frac{1}{k^2} + \frac{1}{j^2} \cdot 4\delta^2(\gamma^*)^2\right) + \left(\frac{1}{k} + \frac{2\delta\gamma^*}{\ell}\right) \cdot 2\delta \\
&\leq \frac{2(\ell - 1)}{k^2} + \frac{\pi^2}{6} \cdot 8\delta^2(\gamma^*)^2 + \frac{2\delta}{k} + \frac{4\delta^2\gamma^*}{\ell} \\
(\text{From our choice of } \ell \text{ and } \delta\gamma^* \leq 0.1\sqrt{\delta/k}) &\leq \frac{2\delta}{k^{1.5}} + \frac{\delta}{k} + \frac{2\delta}{k} + \frac{\sqrt{\delta}}{k} \\
&\leq \frac{2\sqrt{\delta}}{k}.
\end{align*}

\paragraph{Bounding $\|\bb^{< 0}\|_2$.} This is very similar (and in fact slightly simpler) to how we bound $\|\bb^{\geq 0}\|_2$ above; we repeat the argument here for completeness. Let us first sort the coordinates of $\bb^{< 0}$ from smallest to largest entries as $b^{< 0}_{(u^{-1}, \sigma^{-1})}, \dots, b^{< 0}_{(u^{-|U| \times |\Sigma_V|}, \sigma^{-|U| \times |\Sigma_V|})}$ (tie broken arbitrarily). For every $j \leq \min\{C_{< 0}, \ell\}$, consider the sample from Step~\ref{sample:neg-l2-norm} when $S = \{(u^{-1}, \sigma^{-1}), \dots, (u^{-j}, \sigma^{-j})\}$. Since $\bw$ correctly classifies this sample with margin $(1 - \delta)\gamma^*$, we have
\begin{align*}
(1 - \delta)\gamma^* &\leq \left<\bw, 2\gamma^* \be_{\star} + \be_S \hPi^j\right> \\
&= 2\gamma^* \cdot w_{\star} + \bb(\be_S)^T \\
\text{(Observation~\ref{obs:constant-and-non-neg})} &\leq 2\gamma^* \cdot 0.5(1 + \delta) - \left(\sum_{i \in [j]} |b^{< 0}_{(u^{-i}, \sigma^{-i})}|\right).
\end{align*}
Rearranging the above inequality, we have
\begin{align*}
\left(\sum_{i \in [j]} |b^{< 0}_{(u^{-i}, \sigma^{-i})}|\right) \leq 2\delta\gamma^*.
\end{align*}
Recall from our assumption that $|b^{< 0}_{(u^{-1}, \sigma^{-1})}| \geq \cdots \geq |b^{< 0}_{(u^{-j}, \sigma^{-j})}|$. Plugging this into the above, we get
\begin{align} \label{eq:term-by-term-bound-negative}
|b^{< 0}_{(u^{-j}, \sigma^{-j})}| \leq \frac{2\delta \gamma^*}{j}.
\end{align}
Similar to the previous case, although we have derived the above inequality for $j \leq \min\{C_{< 0}, \ell\}$, it also holds for all $j \leq \ell$ simply because $b^{< 0}_{(u^{-j}, \sigma^{-j})} = 0$ for all $j > C_{< 0}$.

We can use this to bound $\|\bb^{< 0}\|_2^2$ as follows.
\begin{align*}
\|\bb^{< 0}\|_2^2 &= \sum_{j=1}^{|U| \cdot |\Sigma_V|} \left(b^{< 0}_{(u^{-j}, \sigma^{-j})}\right)^2 \\
&= \sum_{j < \ell} \left(b^{< 0}_{(u^{-j}, \sigma^{-j})}\right)^2 + \sum_{j \geq \ell} \left(b^{< 0}_{(u^{-j}, \sigma^{-j})}\right)^2 \\
&\leq \sum_{j < \ell} \left(b^{< 0}_{(u^{-j}, \sigma^{-j})}\right)^2 + |b^{< 0}_{(u^{-\ell}, \sigma^{-\ell})}| \cdot \|\bb^{< 0}\|_1 \\
&\overset{\eqref{eq:term-by-term-bound-negative}}{\leq} \sum_{j < \ell} \left(\frac{2\delta\gamma^*}{j}\right)^2 + \left(\frac{2\delta \gamma^*}{\ell}\right) \cdot \|\bb\|_1 \\
&\overset{\eqref{eq:l1-norm-b}}{\leq} \frac{\pi^2}{6} \cdot 4\delta^2(\gamma^*)^2 + \left(\frac{2\delta\gamma^*}{\ell}\right) \cdot 2\delta \\
(\text{From our choice of } \ell \text{ and } \delta\gamma^* \leq 0.1\sqrt{\delta/k}) &\leq \frac{\delta}{k} + \frac{\sqrt{\delta}}{k} \\
&\leq \frac{2\sqrt{\delta}}{k}.
\end{align*}

Using our bounds on $\|\bb^{\geq 0}\|_2^2, \|\bb^{< 0}\|_2^2$, we can easily bound $\|\bb\|_2^2$ by
\begin{align} \label{eq:final-b-norm}
\|\bb\|_2^2 = \|\bb^{\geq 0}\|_2^2 + \|\bb^{< 0}\|_2^2 \leq \frac{4\sqrt{\delta}}{k}.
\end{align}

Next observe that $\|\tPi^j\|_1 = 1$, because each column has exactly a single entry equal to one and the remaining entries equal to zero. Furthermore, $\|\tPi^j\|_\infty = \Delta$ because each row has exactly $\Delta$ entries equal to one\footnote{For every row $(v, \sigma_v)$, these 1-entries are the entries $(u, \sigma_v)$ for all $u \in N(v)$.}. As a result, by Holder's inequality, we have $\|\tPi^j\|_2 \leq \sqrt{\|\tPi^j\|_1 \|\tPi^j\|_\infty} = \sqrt{\Delta}$. From this and from~\eqref{eq:final-b-norm}, we arrive at
\begin{align} \label{eq:b-proj-norm}
\frac{4\sqrt{\delta} \cdot \Delta}{k} \geq \|\bb(\tPi^j)^T\|_2^2 = \|\bv(\Pi^j)^T\|_2^2,
\end{align}
where the latter follows from our definition of $\bb$.

Thus, we have
\begin{align*}
\|\bw (\Pi^j)^T\|_2^2 = \|\bhw (\Pi^j)^T + \bv (\Pi^j)^T\|_2^2 \leq 2\|\bhw(\Pi^j)^T\|_2^2 + 2\|\bv(\Pi^j)^T\|_2^2 \overset{\eqref{eq:b-proj-norm}}{\leq}  2\|\bhw (\Pi^j)^T\|_2^2 + \frac{8\sqrt{\delta} \Delta}{k}.
\end{align*}
Finally, recall from our choice of parameter that $\sqrt{\delta} \Delta \leq 0.1 \sqrt[4]{\delta}$. This, together with the above inequality, implies the desired bound.
\end{proof}

\subsubsection{Decoding Label Cover Assignment}

We now arrive at the last part of the proof, where we show that there exists an assignment that weakly covers at least $\mu = \frac{0.01}{\Delta(\Delta - 1)}$ fraction of vertices in $V$, which completes our soundness proof.

\begin{lemma} 
There exists an assignment $\phi'$ of $\cL$ such that $\wval(\phi') \geq \mu$.
\end{lemma}

\begin{proof}
We define a (random) assignment $\phi$ for $\cL$ as follows:
\begin{itemize}
\item For each $u \in U_{\text{small}}$, let $\phi(u)$ be a random element from $\Sigma_U$ where $\sigma_u \in \Sigma_U$ is selected with probability $\frac{|\hw_{(u, \sigma_u)}|}{\sum_{\sigma \in \Sigma_U} |\hw_{(u, \sigma)}|}$.
\item For each $u \in U_{\text{large}}$, let $\phi(u)$ be an arbitrary element in $\Sigma_U$.
\end{itemize}

We will now argue that $\E_{\phi}[\wval(\phi)] \geq \mu$. Since we assume that $\opt_{(1 - \delta)\gamma^*}^{\cD}(\bw) \leq 0.96(0.25 q)$, we have
\begin{align*}
0.96(0.25 q) &\geq \opt_{(1 - \delta)\gamma^*}^{\cD}(\bw) \\
&\geq (0.25 q) \Pr_{j \in [t], \bs \in \{\pm 1\}^{V \times \Sigma_V}}\left[\left<\bw, \bs\Pi^j\right> < (1 - \delta)\gamma^*\right] \;,
\end{align*}
where the second inequality is due to the error from the samples from Step~\ref{sample:label-cover}.

Let $J \subseteq [t]$ contain all $j \in [t]$ such that $\Pr_{\bs \in \{\pm 1\}^{V \times \Sigma_V}}\left[\left<\bw, \bs\Pi^j\right> < (1 - \delta)\gamma^*\right] < 0.99$. The above inequality implies that
\begin{align} \label{eq:num-good-j-bound}
\Pr_{j \in [t]}\left[j \in J\right] > 0.01.
\end{align}

Now, let us fix $j \in J$.
By definition of $J$, we have 
\begin{align*}
0.01 &\leq \Pr_{\bs \in \{\pm 1\}^{V \times \Sigma_V}}\left[\left<\bw, \bs\Pi^j\right> \geq (1 - \delta)\gamma^*\right] \\
&\leq \Pr_{\bs \in \{\pm 1\}^{V \times \Sigma_V}}\left[|\left<\bw, \bs\Pi^j\right>| \geq (1 - \delta)\gamma^*\right] \\
&= \Pr_{\bs \in \{\pm 1\}^{V \times \Sigma_V}}\left[|\left<\bw(\Pi^j)^T, \bs\right>|^2 \geq ((1 - \delta)\gamma^*)^2\right] \\
(\text{Markov's inequality}) &\leq \frac{\E_{\bs \in \{\pm 1\}^{V \times \Sigma_V}}[|\left<\bw(\Pi^j)^T, \bs\right>|^2]}{((1 - \delta)\gamma^*)^2} \\
&= \frac{\|\bw(\Pi^j)^T\|_2^2}{((1 - \delta)\gamma^*)^2}.
\end{align*}
As a result, we must have $\|\bw(\Pi^j)^T\|_2^2 \geq 0.01((1 - \delta)\gamma^*)^2$. We now apply Lemma~\ref{lem:variance-preserved}, which yields
\begin{align} \label{eq:good-j-norm-bound}
\|\bhw(\Pi^j)^T\|_2^2 \geq 0.005((1 - \delta)\gamma^*)^2 - \frac{\sqrt[4]{\delta}}{k} \geq \frac{2}{k} \;.
\end{align}

Using the definition of $\Pi^j$, we may now rewrite $\|\bhw(\Pi^j)^T\|_2^2$ as follows.
\begin{align}
&\|\bhw(\Pi^j)^T\|_2^2 \nonumber \\
&= \sum_{(v, \sigma_v) \in V \times \Sigma_V} ((\bhw(\Pi^j)^T)_{(v, \sigma_v)})^2 \nonumber \\
&= \sum_{(v, \sigma_v) \in V_j \times \Sigma_V} ((\bhw(\Pi^j)^T)_{(v, \sigma_v)})^2 \nonumber \\
&= \sum_{(v, \sigma_v) \in V_j \times \Sigma_V}\left(\sum_{u \in N(v), \sigma_u \in \pi_{(u, v)}^{-1}(\sigma_v)} \hw_{(u, \sigma_u)}\right)^2 \nonumber \\
&= \sum_{(v, \sigma_v) \in V_j \times \Sigma_V} \sum_{u \in N(v)}\left(\sum_{\sigma_u \in \pi_{(u, v)}^{-1}(\sigma_v)} \hw_{(u, \sigma_u)}\right)^2 \nonumber \\
& \qquad + \sum_{(v, \sigma_v) \in V_j \times \Sigma_V} \sum_{u, u' \in N(v) \atop u \ne u'}\left(\sum_{\sigma_u \in \pi_{(u, v)}^{-1}(\sigma_v)} \hw_{(u, \sigma_u)}\right)\left(\sum_{\sigma_{u'} \in \pi_{(u', v)}^{-1}(\sigma_v)} \hw_{(u', \sigma_{u'})}\right) \nonumber \\
&= \sum_{(v, \sigma_v) \in V_j \times \Sigma_V} \sum_{u \in N(v) \cap U_{\text{small}}}\left(\sum_{\sigma_u \in \pi_{(u, v)}^{-1}(\sigma_v)} \hw_{(u, \sigma_u)}\right)^2 \nonumber \\
& \qquad + \sum_{(v, \sigma_v) \in V_j \times \Sigma_V} \sum_{u, u' \in N(v) \cap U_{\text{small}} \atop u \ne u'}\left(\sum_{\sigma_u \in \pi_{(u, v)}^{-1}(\sigma_v)} \hw_{(u, \sigma_u)}\right)\left(\sum_{\sigma_{u'} \in \pi_{(u', v)}^{-1}(\sigma_v)} \hw_{(u', \sigma_{u'})}\right), \label{eq:separation-norm-intermediate}
\end{align}
where the last equality follows from the fact that $\hw_{(u, \sigma_u)} = 0$ for all $u \notin U_{\text{small}}$.

We will now bound the two terms in~\eqref{eq:separation-norm-intermediate} separately. For the first term, we have
\begin{align}
\sum_{(v, \sigma_v) \in V_j \times \Sigma_V} \sum_{u \in N(v) \cap U_{\text{small}}}\left(\sum_{\sigma_u \in \pi_{(u, v)}^{-1}(\sigma_v)} \hw_{(u, \sigma_u)}\right)^2 
&\leq \sum_{(v, \sigma_v) \in V_j \times \Sigma_V} \sum_{u \in N(v) \cap U_{\text{small}}}\left(\sum_{\sigma_u \in \pi_{(u, v)}^{-1}(\sigma_v)} |\hw_{(u, \sigma_u)}|\right)^2 \nonumber \\
&= \sum_{u \in U_{\text{small}}}\left(\sum_{\sigma_v \in \Sigma_V} \left(\sum_{\sigma_u \in \pi_{(u, v^j(u))}^{-1}(\sigma_v)} |\hw_{(u, \sigma_u)}|\right)^2\right) \nonumber \\
&\leq \sum_{u \in U_{\text{small}}}\left(\sum_{\sigma_u \in \Sigma_U} |\hw_{(u, \sigma_u)}|\right)^2 \nonumber \\
&= \sum_{u \in U_{\text{small}}}  M_u^2 \nonumber \\
&\leq \frac{1}{k} \;, \label{eq:separation-norm-intermediate-first-term}
\end{align}
where the last inequality follows from $M_u \leq 1/k$ for all $u \in U_{\text{small}}$ (by definition) and from $\sum_{u \in U_{\text{small}}} M_u \leq \|\bw\|_1 \leq 1$.

We now move on to bound the second term of~\eqref{eq:separation-norm-intermediate}. To do so, let us observe that, for every $u \in U_{\text{small}}, v \in N(u)$ and $\sigma_v \in \Sigma_V$, we have
\begin{align*}
\Pr_{\phi}[\pi_{(u, v)}(\phi(u)) = \sigma_v] &= \sum_{\sigma_u \in \pi^{-1}_{(u, v)}(\sigma_v)} \frac{|\hw_{(u, \sigma_u)}|}{M_u} \\
&\geq k \sum_{\sigma_u \in \pi^{-1}_{(u, v)}(\sigma_v)} |\hw_{(u, \sigma_u)}| \;.
\end{align*}
As a result, we have
\begin{align}
&\sum_{(v, \sigma_v) \in V_j \times \Sigma_V} \sum_{u, u' \in N(v) \cap U_{\text{small}} \atop u \ne u'}\left(\sum_{\sigma_u \in \pi_{(u, v)}^{-1}(\sigma_v)} \hw_{(u, \sigma_u)}\right)\left(\sum_{\sigma_{u'} \in \pi_{(u', v)}^{-1}(\sigma_v)} \hw_{(u', \sigma_{u'})}\right) \nonumber \\
&\leq \sum_{(v, \sigma_v) \in V_j \times \Sigma_V} \sum_{u, u' \in N(v) \cap U_{\text{small}} \atop u \ne u'} \frac{\Pr_{\phi}[\pi_{(u, v)}(\phi(u)) = \sigma_v]}{k} \cdot \frac{\Pr_{\phi}[\pi_{(u', v)}(\phi(u')) = \sigma_v]}{k} \nonumber \\
&= \frac{1}{k^2} \sum_{(v, \sigma_v) \in V_j \times \Sigma_V} \sum_{u, u' \in N(v) \cap U_{\text{small}} \atop u \ne u'} \Pr_{\phi}[\pi_{(u, v)}(\phi(u)) = \pi_{(u', v)}(\phi(u')) = \sigma_v] \nonumber \\
&= \frac{1}{k^2} \sum_{v \in V_j}  \sum_{u, u' \in N(v) \cap U_{\text{small}} \atop u \ne u'} \sum_{\sigma_v \in \Sigma_V} \Pr_{\phi}[\pi_{(u, v)}(\phi(u)) = \pi_{(u', v)}(\phi(u')) = \sigma_v] \nonumber \\
&= \frac{1}{k^2} \sum_{v \in V_j}  \sum_{u, u' \in N(v) \cap U_{\text{small}} \atop u \ne u'} \Pr_{\phi}[\pi_{(u, v)}(\phi(u)) = \pi_{(u', v)}(\phi(u'))] \nonumber \\
&\leq \frac{1}{k^2} \sum_{v \in V_j}  \sum_{u, u' \in N(v) \cap U_{\text{small}} \atop u \ne u'} \Pr_{\phi}[\phi \text{ weakly covers } v] \nonumber \\
&\leq \frac{\Delta(\Delta - 1)}{k^2} \sum_{v \in V_j} \Pr_{\phi}[\phi \text{ weakly covers } v] \;, \label{eq:separation-norm-intermediate-second-term}
\end{align}
where the last inequality follows from the fact that each $v \in V$ has degree $\Delta$.

Combining~\eqref{eq:good-j-norm-bound},~\eqref{eq:separation-norm-intermediate},~\eqref{eq:separation-norm-intermediate-first-term} and~\eqref{eq:separation-norm-intermediate-second-term}, we have
\begin{align*}
\sum_{v \in V_j} \Pr_{\phi}[\phi \text{ weakly covers } v] \geq \frac{k}{\Delta(\Delta - 1)} \;.
\end{align*}
By summing over all $j \in J$ and using the bound from~\eqref{eq:num-good-j-bound}, we have
\begin{align*}
0.01t \cdot \frac{k}{\Delta(\Delta - 1)} &\leq \sum_{j \in J} \sum_{v \in V_j} \Pr_{\phi}[\phi \text{ weakly covers } v] \\
&\leq \sum_{v \in V} \Pr_{\phi}[\phi \text{ weakly covers } v] \\
&= |V| \cdot \E_{\phi}[\wval(\phi)] \\
&\leq kt \cdot \E_{\phi}[\wval(\phi)] \;.
\end{align*}
Equivalently, this means that $\E_{\phi}[\wval(\phi)] \geq \mu$, which implies that there exists an assignment $\phi'$ of $\cL$ such that $\wval(\phi') \geq \mu$, as desired.
\end{proof}

\section{Conclusions and Open Problems}
\label{sec:open-q}
In this work, we studied the computational complexity of adversarially robust learning of halfspaces in the distribution-independent agnostic PAC model. We provided a simple proper learning algorithm for this problem and a nearly matching computational lower bound.
While proper learners are typically preferable due to their interpretability, the obvious open question is whether significantly faster non-proper learners are possible. We leave this as an interesting open problem. Another direction for future work is to understand the effect of distributional assumptions on the complexity of the problem and to explore the learnability of simple neural networks in this context. 

In addition to the broader open questions posed above, we list several concrete open questions below, regarding our lower bound (Theorem~\ref{thm:running-time-lower-bound}).
\begin{itemize}
\item As alluded to in Section~\ref{sec:open-q}, our proof can only rule out a margin gap $(\gamma, (1 - \nu)\gamma)$ when $\nu > 0$ is a small constant. An intriguing direction here is to extend our hardness to include a larger $\nu$, or conversely give a better algorithm for larger $\nu$. We remark that even the case of margin gap $(\gamma, 0)$ (i.e., $\nu = 1$) remains open for the $L_{\infty}$-margin setting. In this case, the learner only seeks a small misclassification error (without any margin). Note that~\cite{DiakonikolasKM19} gave hardness results that hold even when $\nu = 1$ in the setting of $L_2$-margin.

\item Our technical approach can rule out approximation ratio $\alpha$ of at most 2. The reason is that, the labeled samples (Step~\ref{sample:label-cover} in our reduction) that test the Label Cover constraints are still violated with probability at least 0.5 by the intended solution. As a result, any ``reasonable'' solution will achieve an approximation ratio of $2$. In contrast,~\cite{DiakonikolasKM19} can rule out any constant $\alpha$. Can our hardness be strengthened to also handle larger values of $\alpha$?

\item Finally, it may be interesting to attempt to prove our hardness result under a weaker assumption, specifically ETH. Note that this is open for both our $L_{\infty}$-margin setting and the $L_2$-margin setting in~\cite{DiakonikolasKM19}\footnote{In~\cite{DiakonikolasKM19}, the hardness result is stated under ETH but it is not asymptotically tight (as there is a $\gamma^{o(1)}$ factor in the exponent); their reduction only gives asymptotically tight hardness under Gap-ETH.}. This question is closely related to the general research direction of basing parameterized inapproximability results under ETH instead of Gap-ETH. There are some parameterized hardness of approximation results known under ETH (e.g.,~\cite{Marx13,ChenL19,SLM19,Lin19,BhattacharyyaBE19}), but a large number of questions remain open, including basing Theorem~\ref{thm:label-cover-hardness} on ETH instead of Gap-ETH, which would have given our hardness of $L_{\infty}$-margin learning under ETH. However, it might be possible to give a different proof for hardness of $L_{\infty}$-margin learning assuming ETH directly, without going through such a result as Theorem~\ref{thm:label-cover-hardness}.
\end{itemize}

\bibliographystyle{alpha}
\bibliography{refs}

\newcommand{\etalchar}[1]{$^{#1}$}
\begin{thebibliography}{FGRW12}

\bibitem[ABSS97]{AroraBSS97}
Sanjeev Arora, L{\'{a}}szl{\'{o}} Babai, Jacques Stern, and Z.~Sweedyk.
\newblock The hardness of approximate optima in lattices, codes, and systems of
  linear equations.
\newblock {\em J. Comput. Syst. Sci.}, 54(2):317--331, 1997.

\bibitem[ADV19]{AwasthiDV19}
Pranjal Awasthi, Abhratanu Dutta, and Aravindan Vijayaraghavan.
\newblock On robustness to adversarial examples and polynomial optimization.
\newblock In {\em Advances in Neural Information Processing Systems}, pages
  13737--13747, 2019.

\bibitem[AL88]{AL88}
Dana Angluin and Philip Laird.
\newblock Learning from noisy examples.
\newblock {\em Mach. Learn.}, 2(4):343--370, 1988.

\bibitem[AS18]{AggarwalS18}
Divesh Aggarwal and Noah Stephens{-}Davidowitz.
\newblock {(Gap/S)ETH} hardness of {SVP}.
\newblock In {\em STOC}, pages 228--238, 2018.

\bibitem[Bar75]{Baranyai75}
Zsolt Baranyai.
\newblock On the factorization of the complete uniform hypergraph.
\newblock {\em Infinite and Finite Sets, Proc. Coll. Keszthely}, 10:91--107,
  1975.

\bibitem[BB14]{BalcanB14}
Maria{-}Florina Balcan and Christopher Berlind.
\newblock A new perspective on learning linear separators with large $l_q l_p$
  margins.
\newblock In {\em AISTATS}, pages 68--76, 2014.

\bibitem[BBE{\etalchar{+}}19]{BhattacharyyaBE19}
Arnab Bhattacharyya, {\'{E}}douard Bonnet, L{\'{a}}szl{\'{o}} Egri, Suprovat
  Ghoshal, {Karthik {C. S.}}, Bingkai Lin, Pasin Manurangsi, and D{\'{a}}niel
  Marx.
\newblock Parameterized intractability of even set and shortest vector problem.
\newblock {\em Electronic Colloquium on Computational Complexity {(ECCC)}},
  26:115, 2019.

\bibitem[BCM{\etalchar{+}}13]{BiggioCMNSLGR13}
Battista Biggio, Igino Corona, Davide Maiorca, Blaine Nelson, Nedim Srndic,
  Pavel Laskov, Giorgio Giacinto, and Fabio Roli.
\newblock Evasion attacks against machine learning at test time.
\newblock In {\em ECML PKDD}, pages 387--402, 2013.

\bibitem[Ber41]{Berry41}
Andrew~C. Berry.
\newblock The accuracy of the {G}aussian approximation to the sum of
  independent variates.
\newblock {\em Transactions of the American Mathematical Society},
  49(1):122--136, 1941.

\bibitem[BGKM18]{BhattacharyyaGS18}
Arnab Bhattacharyya, Suprovat Ghoshal, {Karthik {C. S.}}, and Pasin Manurangsi.
\newblock Parameterized intractability of even set and shortest vector problem
  from {Gap-ETH}.
\newblock In {\em ICALP}, pages 17:1--17:15, 2018.

\bibitem[BGS17]{BennettGS17}
Huck Bennett, Alexander Golovnev, and Noah Stephens{-}Davidowitz.
\newblock On the quantitative hardness of {CVP}.
\newblock In {\em FOCS}, pages 13--24, 2017.

\bibitem[BLPR19]{BubeckLPR19}
Sebastien Bubeck, Yin-Tat Lee, Eric Price, and Ilya~P. Razenshteyn.
\newblock Adversarial examples from computational constraints.
\newblock In {\em ICML}, pages 831--840, 2019.

\bibitem[BM02]{BartlettM02}
Peter~L. Bartlett and Shahar Mendelson.
\newblock Rademacher and gaussian complexities: Risk bounds and structural
  results.
\newblock {\em Journal of Machine Learning Research}, 3:463--482, 2002.

\bibitem[BS00]{BenDavidS00}
Shai Ben{-}David and Hans~Ulrich Simon.
\newblock Efficient learning of linear perceptrons.
\newblock In {\em Advances in Neural Information Processing Systems}, pages
  189--195, 2000.

\bibitem[BS12]{BirnbaumS12}
Aharon Birnbaum and Shai Shalev{-}Shwartz.
\newblock Learning halfspaces with the zero-one loss: Time-accuracy tradeoffs.
\newblock In {\em Advances in Neural Information Processing Systems}, pages
  935--943, 2012.

\bibitem[CBM18]{CullinaBM18}
Daniel Cullina, Arjun~Nitin Bhagoji, and Prateek Mittal.
\newblock Pac-learning in the presence of adversaries.
\newblock In {\em Advances in Neural Information Processing Systems}, pages
  228--239, 2018.

\bibitem[CCK{\etalchar{+}}17]{ChalermsookCKLM17}
Parinya Chalermsook, Marek Cygan, Guy Kortsarz, Bundit Laekhanukit, Pasin
  Manurangsi, Danupon Nanongkai, and Luca Trevisan.
\newblock From {Gap-ETH} to {FPT}-inapproximability: Clique, dominating set,
  and more.
\newblock In {\em FOCS}, pages 743--754, 2017.

\bibitem[CGK{\etalchar{+}}19]{Cohen-AddadG0LL19}
Vincent Cohen{-}Addad, Anupam Gupta, Amit Kumar, Euiwoong Lee, and Jason Li.
\newblock Tight {FPT} approximations for k-median and k-means.
\newblock In {\em ICALP}, pages 42:1--42:14, 2019.

\bibitem[CL19]{ChenL19}
Yijia Chen and Bingkai Lin.
\newblock The constant inapproximability of the parameterized dominating set
  problem.
\newblock {\em {SIAM} J. Comput.}, 48(2):513--533, 2019.

\bibitem[CM18]{KM-tutorial}
Zico Colter and Aleksander Madry.
\newblock Adversarial robustness - theory and practice.
\newblock NeurIPS 2018 tutorial, available at
  https://adversarial-ml-tutorial.org/, 2018.

\bibitem[Din16]{Dinur16}
Irit Dinur.
\newblock Mildly exponential reduction from gap {3SAT} to polynomial-gap
  label-cover.
\newblock {\em Electronic Colloquium on Computational Complexity {(ECCC)}},
  23:128, 2016.

\bibitem[DKM19]{DiakonikolasKM19}
Ilias Diakonikolas, Daniel Kane, and Pasin Manurangsi.
\newblock Nearly tight bounds for robust proper learning of halfspaces with a
  margin.
\newblock In {\em Advances in Neural Information Processing Systems}, pages
  10473--10484, 2019.

\bibitem[DM18]{DinurM18}
Irit Dinur and Pasin Manurangsi.
\newblock {ETH}-hardness of approximating 2-{CSPs} and directed steiner
  network.
\newblock In {\em ITCS}, pages 36:1--36:20, 2018.

\bibitem[DNV19]{DegwekarNV19}
Akshay Degwekar, Preetum Nakkiran, and Vinod Vaikuntanathan.
\newblock Computational limitations in robust classification and win-win
  results.
\newblock In {\em COLT}, pages 994--1028, 2019.

\bibitem[DS14]{DinurS14}
Irit Dinur and David Steurer.
\newblock Analytical approach to parallel repetition.
\newblock In {\em STOC}, pages 624--633, 2014.

\bibitem[Ess42]{Esseen42}
Carl-Gustav Esseen.
\newblock On the {L}iapunoff limit of error in the theory of probability.
\newblock {\em Arkiv f{\"o}r matematik, astronomi och fysik}, A28:1--19, 1942.

\bibitem[Fei98]{Feige98}
Uriel Feige.
\newblock A threshold of ln \emph{n} for approximating set cover.
\newblock {\em J. {ACM}}, 45(4):634--652, 1998.

\bibitem[FGKP06]{FeldmanGKP06}
Vitaly Feldman, Parikshit Gopalan, Subhash Khot, and Ashok~Kumar Ponnuswami.
\newblock New results for learning noisy parities and halfspaces.
\newblock In {\em FOCS}, pages 563--574, 2006.

\bibitem[FGRW12]{FeldmanGRW12}
Vitaly Feldman, Venkatesan Guruswami, Prasad Raghavendra, and Yi~Wu.
\newblock Agnostic learning of monomials by halfspaces is hard.
\newblock {\em {SIAM} J. Comput.}, 41(6):1558--1590, 2012.

\bibitem[FS97]{FreundSchapire:97}
Yoav Freund and Robert Schapire.
\newblock A decision-theoretic generalization of on-line learning and an
  application to boosting.
\newblock {\em J. Comput. Syst. Sci.}, 55(1):119--139, 1997.

\bibitem[Gen01a]{Gentile01}
Claudio Gentile.
\newblock A new approximate maximal margin classification algorithm.
\newblock {\em J. Mach. Learn. Res.}, 2:213--242, 2001.

\bibitem[Gen01b]{Gentile:01}
Claudio Gentile.
\newblock A new approximate maximal margin classification algorithm.
\newblock {\em Journal of Machine Learning Research}, 2:213--242, 2001.

\bibitem[Gen03]{Gentile03a}
Claudio Gentile.
\newblock The robustness of the p-norm algorithms.
\newblock {\em Mach. Learn.}, 53(3):265--299, 2003.

\bibitem[GLS01]{GroveLS01}
Adam~J. Grove, Nick Littlestone, and Dale Schuurmans.
\newblock General convergence results for linear discriminant updates.
\newblock {\em Mach. Learn.}, 43(3):173--210, 2001.

\bibitem[GR09]{GuruswamiR09}
Venkatesan Guruswami and Prasad Raghavendra.
\newblock Hardness of learning halfspaces with noise.
\newblock {\em {SIAM} J. Comput.}, 39(2):742--765, 2009.

\bibitem[GSS15]{GoodfellowSS14}
Ian~J. Goodfellow, Jonathon Shlens, and Christian Szegedy.
\newblock Explaining and harnessing adversarial examples.
\newblock In {\em ICLR}, 2015.

\bibitem[H{\aa}s96]{Hastad96}
Johan H{\aa}stad.
\newblock Clique is hard to approximate within $n^{1 - \epsilon}$.
\newblock In {\em FOCS}, pages 627--636, 1996.

\bibitem[H{\aa}s01]{Hastad01}
Johan H{\aa}stad.
\newblock Some optimal inapproximability results.
\newblock {\em J. {ACM}}, 48(4):798--859, 2001.

\bibitem[Hau92]{Haussler:92}
David Haussler.
\newblock {Decision theoretic generalizations of the PAC model for neural net
  and other learning applications}.
\newblock {\em Information and Computation}, 100:78--150, 1992.

\bibitem[IP01]{IP01}
Russell Impagliazzo and Ramamohan Paturi.
\newblock On the complexity of k-{SAT}.
\newblock {\em J. Comput. Syst. Sci.}, 62(2):367--375, 2001.

\bibitem[IPZ01]{IPZ01}
Russell Impagliazzo, Ramamohan Paturi, and Francis Zane.
\newblock Which problems have strongly exponential complexity?
\newblock {\em J. Comput. Syst. Sci.}, 63(4):512--530, 2001.

\bibitem[JKR19]{JainKR19}
Vishesh Jain, Frederic Koehler, and Andrej Risteski.
\newblock Mean-field approximation, convex hierarchies, and the optimality of
  correlation rounding: a unified perspective.
\newblock In {\em STOC}, pages 1226--1236, 2019.

\bibitem[KLM19]{SLM19}
{Karthik {C. S.}}, Bundit Laekhanukit, and Pasin Manurangsi.
\newblock On the parameterized complexity of approximating dominating set.
\newblock {\em J. {ACM}}, 66(5):33:1--33:38, 2019.

\bibitem[KP02]{Koltchinskii2002}
Vladimir Koltchinskii and Dmitry Panchenko.
\newblock Empirical margin distributions and bounding the generalization error
  of combined classifiers.
\newblock {\em Ann. Statist.}, 30(1):1--50, 02 2002.

\bibitem[KSS94]{KSS:94}
Michael Kearns, Robert Schapire, and Linda Sellie.
\newblock {Toward Efficient Agnostic Learning}.
\newblock {\em Machine Learning}, 17(2/3):115--141, 1994.

\bibitem[KST08]{KakadeST08}
Sham~M. Kakade, Karthik Sridharan, and Ambuj Tewari.
\newblock On the complexity of linear prediction: Risk bounds, margin bounds,
  and regularization.
\newblock In {\em Advances in Neural Information Processing Systems}, pages
  793--800, 2008.

\bibitem[Lin19]{Lin19}
Bingkai Lin.
\newblock A simple gap-producing reduction for the parameterized set cover
  problem.
\newblock In {\em ICALP}, pages 81:1--81:15. Schloss Dagstuhl - Leibniz-Zentrum
  f{\"{u}}r Informatik, 2019.

\bibitem[Lit87]{Lit87}
Nick Littlestone.
\newblock Learning quickly when irrelevant attributes abound: A new
  linear-threshold algorithm.
\newblock {\em Machine Learning}, 2(4):285--318, 1987.

\bibitem[LMS11]{LokshtanovMS11}
Daniel Lokshtanov, D{\'{a}}niel Marx, and Saket Saurabh.
\newblock Lower bounds based on the exponential time hypothesis.
\newblock {\em Bulletin of the {EATCS}}, 105:41--72, 2011.

\bibitem[LS11]{LS:11malicious}
Phil Long and Rocco Servedio.
\newblock Learning large-margin halfspaces with more malicious noise.
\newblock {\em Advances in Neural Information Processing Systems}, 2011.

\bibitem[Man20]{M20}
Pasin Manurangsi.
\newblock Tight running time lower bounds for strong inapproximability of
  maximum \emph{k}-coverage, unique set cover and related problems (via
  \emph{t}-wise agreement testing theorem).
\newblock In {\em SODA}, pages 62--81, 2020.

\bibitem[Mar13]{Marx13}
D{\'{a}}niel Marx.
\newblock Completely inapproximable monotone and antimonotone parameterized
  problems.
\newblock {\em J. Comput. Syst. Sci.}, 79(1):144--151, 2013.

\bibitem[MGDS20]{MonGDS20}
Omar Montasser, Surbhi Goel, Ilias Diakonikolas, and Nathan Srebro.
\newblock Efficiently learning adversarially robust halfspaces with noise.
\newblock {\em CoRR}, abs/2005.07652, 2020.

\bibitem[MHS19]{MontasserHS19}
Omar Montasser, Steve Hanneke, and Nathan Srebro.
\newblock {VC} classes are adversarially robustly learnable, but only
  improperly.
\newblock In {\em COLT}, pages 2512--2530, 2019.

\bibitem[MR10]{MoshkovitzR10}
Dana Moshkovitz and Ran Raz.
\newblock Two-query {PCP} with subconstant error.
\newblock {\em J. {ACM}}, 57(5):29:1--29:29, 2010.

\bibitem[MR17]{MR17}
Pasin Manurangsi and Prasad Raghavendra.
\newblock A birthday repetition theorem and complexity of approximating dense
  {CSP}s.
\newblock In {\em ICALP}, pages 78:1--78:15, 2017.

\bibitem[Raz98]{Raz98}
Ran Raz.
\newblock A parallel repetition theorem.
\newblock {\em {SIAM} J. Comput.}, 27(3):763--803, 1998.

\bibitem[Ros58]{Rosenblatt:58}
Frank Rosenblatt.
\newblock The {P}erceptron: a probabilistic model for information storage and
  organization in the brain.
\newblock {\em Psychological Review}, 65:386--407, 1958.

\bibitem[SSS09]{SSS09}
Shai Shalev{-}Shwartz, Ohad Shamir, and Karthik Sridharan.
\newblock Agnostically learning halfspaces with margin errors.
\newblock In {\em Technical report, Toyota Technological Institute}, 2009.

\bibitem[SSS10]{SSS10}
Shai Shalev{-}Shwartz, Ohad Shamir, and Karthik Sridharan.
\newblock Learning kernel-based halfspaces with the zero-one loss.
\newblock In {\em COLT}, pages 441--450, 2010.

\bibitem[SST{\etalchar{+}}18]{SchmidtSTTM18}
Ludwig Schmidt, Shibani Santurkar, Dimitris Tsipras, Kunal Talwar, and
  Aleksander Madry.
\newblock Adversarially robust generalization requires more data.
\newblock In {\em Advances in Neural Information Processing Systems}, pages
  5019--5031, 2018.

\bibitem[SZS{\etalchar{+}}14]{SzegedyZSBEGF13}
Christian Szegedy, Wojciech Zaremba, Ilya Sutskever, Joan Bruna, Dumitru Erhan,
  Ian~J. Goodfellow, and Rob Fergus.
\newblock Intriguing properties of neural networks.
\newblock In {\em ICLR}, 2014.

\bibitem[Vap98]{Vapnik:98}
Vladimir Vapnik.
\newblock {\em Statistical Learning Theory}.
\newblock Wiley-Interscience, New York, 1998.

\end{thebibliography}

\appendix

\section{Proof of Fact~\ref{fact:adversarial-robustness-vs-margin}}
\label{sec:adversarial-v-margin}

\begin{proof}[Proof of Fact~\ref{fact:adversarial-robustness-vs-margin}]
For convenience, let $\bw' = \frac{\bw}{\|\bw\|_q}$.
Consider any $(\bx, y) \in \R^d \times \{\pm 1\}$. We claim that $\sgn(\left<\bw', \bx\right> - \gamma) \ne y$ iff $\exists \bz \in \cU_{p, \gamma}(\bx), h_{\bw}(\bz) \ne y$. Below we only show this statement when $y = -1$. The case $y = 1$ follows analogously.

Suppose $y = -1$. Let us first prove the forward direction: if $\sgn(\left<\bw', \bx\right> - \gamma) \ne y = -1$, we have $\left<\bw', \bx\right> \geq \gamma$. Let $\bt \in \R^d$ be such that $t_i = \gamma \cdot \sgn(w'_i) \cdot |w'_i|^{q - 1}$. It is simple to verify that $\|\bt\|_p = \gamma$ and that $\left<\bw', \bt\right> = \gamma$. Consider $\bz = \bx - \bt \in \cU_{p, \gamma}(\bx)$. We have
\begin{align*}
\left<\bw', \bz\right> = \left<\bw', \bx\right> - \left<\bw', \bt\right> \geq 0.
\end{align*}
Thus, we have $h_{\bw}(\bz) = h_{\bw'}(\bz) = 1 \ne y$ as desired.

We will next prove the converse by contrapositive.
Suppose that $\sgn(\left<\bw', \bx\right> - \gamma) = y = -1$. Then, we have $\left<\bw', \bx\right> < -\gamma$ and, for any $\bz \in \cU_{p, \gamma}(\bx)$, we can derive
\begin{align*}
\left<\bw', \bz\right>
&\leq \left<\bw', \bx\right> + |\left<\bw', \bz - \bx\right>| \\
(\text{Holder's Inequality}) &< -\gamma + \|\bw'\|_q \|\bz - \bx\|_p \\
&< 0 \;,
\end{align*}
where the last inequality follows from $\|\bw'\|_q = 1$ and $\|\bz - \bx\|_p \leq \gamma$.  Hence, $h_{\bw}(\bz) = h_{\bw'}(\bz) = -1 = y$ as desired.

To summarize, so far we have shown that $\sgn(\left<\bw', \bx\right> - \gamma) \ne y$ iff $\exists \bz \in \cU_{p, \gamma}(\bx), h_{\bw}(\bz) \ne y$. As a result, we have
\begin{align*}
\mathcal{R}_{\mathcal{U}_{p,\gamma}}(h_{\bw}, \mathcal{D}) &= \Pr_{(\bx, y) \sim \cD}\left[\exists \bz \in \cU_{p, \gamma}(\bx), h_{\bw}(\bz) \ne y\right] \\
&= \Pr_{(\bx, y) \sim \cD}\left[\sgn(\left<\bw', \bx\right> - \gamma) \ne y\right] \\
&= \err_{\gamma}^{\cD}(\bw') \;. \qedhere
\end{align*}
\end{proof}

\section{On the Necessity of Bicriterion Approximation} \label{app:agnostic-hard}

In this section, we briefly argue that, when there is no margin gap (i.e., for $\nu = 0$), the learning problem we consider is computationally hard. In particular, we show the following hardness that, when $\nu = 0$ and\footnote{We remark that 0.5 is unimportant here and the reduction works for any $\gamma \leq 0.5$.} $\gamma = 0.5$, there is no $\poly(d/\eps)$-time learning algorithm for any constant approximation ratio $\alpha > 1$. Note that this result holds under the assumption $NP \nsubseteq RP$. If we further assume ETH, we can get a stronger lower bound of $2^{(d/\eps)^c}$ for some constant $c > 0$. This is in contrast to our main algorithmic result (Theorem~\ref{thm:learning-algo-main}) that, when $\nu, \gamma > 0$ and $\alpha > 1$ are constants, runs in polynomial (in $d/\eps$) time.

\begin{proposition}
For any constant $\alpha > 1$, assuming $NP \nsubseteq RP$, there is no proper 0-robust $\alpha$-agnostic learner for $L_{\infty}$-$0.5$-margin halfspaces in time $\poly(d/\eps)$. 
\end{proposition}

Similar to before (see, e.g., Section~\ref{sec:hardness}), the above result immediately follows from Lemma~\ref{lem:no-margin-gap-hardness} below. We will henceforth focus on the proof of this lemma.

\begin{lemma} \label{lem:no-margin-gap-hardness}
For any constant $\alpha > 1$, assuming $P \ne NP$, no $\poly(d/\eps)$-time algorithm can, given $\eps > 0$ and a multiset $S \subseteq \bB^d_{\infty} \times \{\pm 1\}$ of labeled samples, distinguish between:
\begin{itemize}
\item (Completeness) $\opt_{0.5}^S \leq \eps$.
\item (Soundness) $\opt_{0.5}^S > \alpha \cdot \eps$.
\end{itemize}
\end{lemma}

To prove Lemma~\ref{lem:no-margin-gap-hardness}, we will use the following hardness for (no-margin) proper agnostic learning of halfspaces. Observe here that in the Completeness case, there is an extra promise that every coordinate of $w$ is non-negative; this follows from the construction of~\cite{AroraBSS97}.

\begin{theorem}[\cite{AroraBSS97}] \label{thm:hardness-agnostic}
For any constant $\alpha > 1$, assuming $P \ne NP$, no $\poly(\td/\teps)$-time algorithm can, given $\teps > 0$ and a multiset $\tS \subseteq \bB^{\td}_{\infty} \times \{\pm 1\}$ of labeled samples, distinguish between:
\begin{itemize}
\item (Completeness) There exists $\tbw \in \bB^{\td}_1$ where $\tilde{w}_i \geq 0$ for all $i \in [d]$ such that $\err_0^{\tS}(\tbw) \leq \teps$.
\item (Soundness) $\opt_0^{\tS} > \alpha \cdot \teps$.
\end{itemize}
\end{theorem}

\begin{proof}[Proof of Lemma~\ref{lem:no-margin-gap-hardness}]
Given a multiset $\tS \subseteq \bB^{\td}_{\infty} \times \{\pm 1\}$ from Theorem~\ref{thm:hardness-agnostic}. Let $m = |\tS|$. We create a new multiset of samples $S \subseteq \bB^{d}_{\infty} \times \{\pm 1\}$ as follows:
\begin{itemize}
\item Let $d = \td + 1$.
\item For every $(\bx, y) \in \tS$, add\footnote{Note that we use $\bx \circ y$ to denote the vector resulting from concatenating $\bx$ and $y$.} $(\bx \circ y, y)$ to $S$.
\item Add $\lceil \alpha m + 1\rceil$ copies of $((1, \dots, 1, 0), +1)$ to $S$.
\end{itemize}
Finally, let $\eps = \frac{\teps \cdot m}{m + \lceil \alpha m + 1\rceil}$. It is obvious that the reduction runs in polynomial time. We will now argue its completeness and soundness.

\paragraph{Completeness.} Suppose that there is $\tbw \in \bB^{\td}_1$ whose coordinates are non-negative such that $\err_0^{\tS}(\tbw) \leq \teps$. Consider $\bw = (0.5\tbw / \|\tbw\|_1) \circ 0.5$. Since each coordinate of $\tbw$ is non-negative, the new halfspace $\bw$ correctly classifies the last sample with margin 0.5. Furthermore, it is also simple to verify that $(\bx, y) \in \tS$ is correctly classified by $\tbw$ (with margin 0) iff $(\bx \circ y, y)$ is correctly classified by $\bw$ with margin $0.5$. As a result, we have $\err^S_{0.5}
(\bw) = \frac{m}{m + \lceil \alpha m + 1\rceil} \cdot \err_0^{\tS}(\tbw) \leq \eps$, as desired.

\paragraph{Soundness.} Suppose that $\opt_0^{\tS} > \alpha \cdot \teps$. Consider any $\bw \in \bB^{\td}_1$. Let us consider two cases, based on the value of $w_{d + 1}$.
\begin{itemize}
\item $w_{d + 1} > 1/2$. In this case, $\left<\bw, (1, \dots, 1, 0) \right> < 0.5$. In other words, $\bw$ does \emph{not} correctly classify the last sample with margin 0.5. As a result, we immediately have $\err_{0.5}^{S}(\bw) \geq \frac{\lceil \alpha m + 1 \rceil}{m + \lceil \alpha m + 1 \rceil} > \alpha \cdot \epsilon$ as desired.
\item $w_{d + 1} \leq 1/2$. In this case, notice that $\sgn(\left<\bw, \bx \circ y\right> - 0.5y) = y$ implies that $\sgn(\left<(w_1, \dots, w_d), \bx\right>) = y$. Thus, we have $\err_{0.5}^S(\bw) \geq \frac{m}{m + \lceil \alpha m + 1\rceil} \cdot \err_0^{\tS}((w_1, \dots, w_d)) \geq \frac{m}{m + \lceil \alpha m + 1\rceil} \cdot (\alpha\cdot\teps) = \alpha \cdot \eps$.
\end{itemize}
Hence, in both cases, we have $\opt_{0.5}^S > \alpha \cdot \eps$, which concludes our proof.
\end{proof}

\end{document}